\documentclass{article}

\usepackage{arxiv}

\usepackage[utf8]{inputenc} 
\usepackage[T1]{fontenc}    
\usepackage{hyperref}       
\usepackage{url}            
\usepackage{booktabs}       
\usepackage{amsfonts}       
\usepackage{nicefrac}       
\usepackage{microtype}      
\usepackage{xcolor}         
\usepackage{enumitem}
\usepackage{wrapfig}
\usepackage{xr}
\usepackage{subfiles}
\usepackage{caption}
\setlist{leftmargin=5.5mm}

\makeatletter
\newcommand*{\addFileDependency}[1]{
  \typeout{(#1)}
  \@addtofilelist{#1}
  \IfFileExists{#1}{}{\typeout{No file #1.}}
}
\makeatother

\title{Hybrid Value Estimation for Off-policy Evaluation and Offline Reinforcement Learning}

\author{%
  Xue-Kun Jin\textsuperscript{\rm 1,}\thanks{Equal Contribution $^\dagger$Corresponding Author}~~, Xu-Hui Liu\textsuperscript{\rm 1,}\footnotemark[1]~~, Shengyi Jiang\textsuperscript{\rm 2}, Yang Yu\textsuperscript{\rm 1,\rm 3, $\dagger$}\\
  \textsuperscript{\rm 1} National Key Laboratory for Novel Software Technology, Nanjing University, Nanjing, China \\
  \textsuperscript{\rm 2} Department of Computer Science, The University of Hong Kong, Hong Kong SAR, China \\
  \textsuperscript{\rm 3} Polixir.ai\\
  \texttt{\{jinxk, liuxh\}@lamda.nju.edu.cn, syjiang@cs.hku.hk, yuy@nju.edu.cn}
}
\date{}

\usepackage{amsmath}
\usepackage{amssymb}
\usepackage{mathtools}
\usepackage{amsthm}

\usepackage[capitalize,noabbrev]{cleveref}

\theoremstyle{plain}
\newtheorem{theorem}{Theorem}[section]
\newtheorem{proposition}[theorem]{Proposition}
\newtheorem{lemma}[theorem]{Lemma}

\theoremstyle{definition}
\newtheorem{definition}[theorem]{Definition}

\theoremstyle{remark}

\DeclareMathOperator*{\argmax}{arg\,max}
\DeclareMathOperator*{\argmin}{arg\,min}

\usepackage[textsize=tiny]{todonotes}

\usepackage{algorithm}
\usepackage{algorithmic}
\usepackage[american]{babel}
\usepackage{booktabs}
\usepackage{multirow}
\usepackage{graphicx} 
\usepackage{subfigure}

\begin{document}

\maketitle

\begin{abstract}
Value function estimation is an indispensable subroutine in reinforcement learning, which becomes more challenging in the offline setting. In this paper, we propose Hybrid Value Estimation (HVE) to reduce value estimation error, which trades off bias and variance by balancing between the value estimation from offline data and the learned model. Theoretical analysis discloses that HVE enjoys a better error bound than the direct methods. HVE can be leveraged in both off-policy evaluation and offline reinforcement learning settings. We, therefore, provide two concrete algorithms Off-policy HVE (OPHVE) and Model-based Offline HVE (MOHVE), respectively. Empirical evaluations on MuJoCo tasks corroborate the theoretical claim. OPHVE outperforms other off-policy evaluation methods in all three metrics measuring the estimation effectiveness, while MOHVE achieves better or comparable performance with state-of-the-art offline reinforcement learning algorithms. We hope that HVE could shed some light on further research on reinforcement learning from fixed data.
\end{abstract}

\section{Introduction}
\label{submission}
Reinforcement learning (RL)~\cite{SuttonB98, dqn, trpo, remern} has demonstrated great success in various sequential decision making problems, e.g., sequential recommendation systems~\cite{rs1, rs2} and robotic locomotion skill learning~\cite{rc1, sac}. Current reinforcement learning algorithms fall into two major categories: online RL algorithms and offline RL algorithms. The former category requires iteratively collecting samples by interacting with the environment and usually achieves optimal performance. The latter category learns a policy from a fixed dataset, making it suitable for risk-sensitive applications. However, the learned policy may still be sub-optimal. In both online setting and offline setting, value function estimation is an indispensable subroutine, which can be seen as the precondition of policy evaluation and policy improvement. 

However, value function estimation was found to perform poorly without online interaction \cite{neorl}, which is an important source of the sub-optimality in current offline RL algorithms. This failure is generally attributed to the extrapolation error when the value function is evaluated on out-of-distribution actions~\cite{bcq, BEAR}. Model-based methods \cite{shi.aaai19,shang.mlj21,maple} are a natural choice to approach this problem for two reasons. First, models are mainly trained with supervised learning with a fixed target, e.g., \cite{DBLP:conf/nips/HaS18}, which is more stable than bootstrapping. Second, dynamics tend to be simpler than value function, so the learned model enjoys a better smoothness and generalization ability~\cite{experssivity}. 

Based on these existing promising benefits of model-based value estimation, we expect to improve it further. There are three kinds of errors in vanilla model-based value estimation, which are optimization error, projection error and statistical error. They are caused by the limitation of optimization algorithms, expressivity of neural networks and quantity of data. Such errors will accumulate along the trajectory and lead to the so-called compounding error~\cite{mbpo,xu.nips20}, which significantly degrades the performance of model-based value estimation. Inspired by the fact that the data in the offline dataset does not suffer from model error, we use the offline data to correct the error in value estimation caused by the inaccurate model.

The proposed value estimation method is named Hybrid Value Estimation (HVE), which uses both model-generated data and offline dataset to perform estimation. Model induces bias because of model error while the offline dataset induces bias and variance because of the policy divergence and data limitation respectively. HVE achieves a trade-off between bias and variance by automatically choosing the step length in Bellman update, which balances value estimation from offline data and the learned model. We implement two concrete algorithms OPHVE and MOHVE that apply HVE in off-policy evaluation and offline reinforcement learning respectively.

We provide a theoretical proof that HVE has a better error bound than vanilla model-based value estimation. The results on OPE benchmark DOPE~\cite{dope} verify this theory empirically.   The experiments suggest that MOHVE outperforms previous offline RL algorithms~\cite{mopo, combo, CQL, td3bc} on the offline RL benchmark D4RL~\cite{d4rl}. We also design an experiment to analyze the effect of step length. The results show its significance in the performance of OPHVE and MOHVE, and our automatic step length adjustment mechanism can find the near-optimal solution. Furthermore, we integrate MOHVE with OPHVE, improving the training and evaluation pipeline of offline RL.

\section{Background and Related Work}
Consider the standard Markov decision process (MDP), in which the environment is defined by a tuple $M = (\mathcal{S}, \mathcal{A}, T, r, \rho_0, \gamma)$.
$\mathcal{S}$ and $\mathcal{A}$ denote the state space and action space respectively, both of which can be continuous or discrete. 
$T(s'|s,a)$ is the dynamics or transition distribution and $\rho_0(s)$ is the initial state distribution. 
$r(s,a)\in [0,R_{\max}]$ denotes the reward function, and $\gamma \in (0,1)$ the discount factor. The goal of reinforcement learning is to find the optimal policy $\pi^*(a|s)$ that maximizes the expected discounted cumulative rewards:
\begin{equation}
    \arg\max\limits_{\pi} J_M(\pi) := \mathop{\mathbb{E}} \limits_{\pi,T,\rho_0}\left[\sum_{t=0}^\infty \gamma^t r\left(s_t,a_t\right)\right]\,.
\end{equation}
Given a state-action pair $(s,a)$, define the state-action value function $Q^\pi(s,a)$ as $\mathop{\mathbb{E}}_{\pi,T}[\sum_{t=0}^\infty \gamma^t r(s_t,a_t) | s_0 = s, a_0 = a]$, which represents the expected discounted return starting from $(s,a)$.
Similarly, define state value function $V^\pi(s)$ as $\mathop{\mathbb{E}}_{\pi,T}[\sum_{t=0}^\infty \gamma^t r(s_t,a_t) | s_0 = s]$, the expected discounted return under $\pi$ when starting from state $s$. To facilitate our proof, we denote the discounted state visitation distribution of a policy $\pi$ under MDP $M$ as $d^\pi_{M}(s):=(1-\gamma) \sum_{t=0}^\infty \gamma^t \mathcal{P}(s_t=s|\pi)$, where $\mathcal{P}(s_t=s|\pi)$ is the probability of reaching state $s$ at time $t$ by rolling out $\pi$ in $M$. 
Furthermore, we denote the discounted state-action visitation distribution $d^\pi_M(s,a):=d^\pi_M(s)\pi(a|s)$.

\textbf{Off-policy Evaluation. }OPE aims to evaluate a set of policies using data collected by other behavior policies without interacting with the environment. The behavior policies are sometimes known.

A straightforward approach to OPE is to fit the Q function directly from data using standard approximate dynamic programming (ADP) techniques, e.g., Fitted Q Evaluation (FQE)~\cite{fqe, fqe1}. 
Such an evaluation incurs high bias. Importance Sampling (IS) performs evaluation by calculating the distribution ratio of trajectories generated by behavior policy and target policy~\cite{eligibility_traces, HallakM17}. Although being consistent and unbiased in theory, this class of estimators suffers from high variance. Marginalized Importance Sampling (MIS)~\cite{dualdice, gendice,mwl, bestDICE, sope} reduces the variance of IS by reweighting each state-action pair rather than reweighting the entire trajectory. Doubly Robust (DR)~\cite{doubly_roubst, MAGIC} leverages an approximate Q function to decrease the variance of the unbiased estimates provided by IS, where MAGIC~\cite{MAGIC} uses a model to obtain the approximate Q function and aims to trade off between bias and variance based on MSE minimization. However, the trade-off needs the knowledge of true value function, which is not available in OPE setting. OPHVE achieves bias-variance trade-off by minimizing Mean Absolute Error and does not require the knowledge of true value function. 

\textbf{Model-based RL. } Model-based reinforcement learning is featured with learned dynamic models. Recent model-based RL algorithms~\cite{mbpo, mopo, combo} use a bootstrap ensemble of dynamic models. The learned dynamic transition and (learned) reward function establish a model MDP $\widehat{M}=(\mathcal{S}, \mathcal{A}, \widehat{T}, \hat{r}, \rho_0, \gamma)$. The agent can interact with $\widehat M$ to generate more data. Subsequently, we can use any learning or planning algorithms~\cite{mve, pets} to recover the optimal policy in the model MDP as $\hat{\pi} = \arg \max_\pi J_{\widehat{M}}(\pi)$. Besides generating more data, models are also useful for learning transferable policies \cite{luo.aaai22}.

\textbf{Offline RL. } In the offline reinforcement learning setting, agents cannot interact with a real environment like online RL. Instead, it only has access to a static dataset $\mathcal{D}_{\text{env}} = \{(s, a, r, s')\}$ collected beforehand by one or a mixture of behavior policies, denoted by $\pi_\beta$.

Most recent offline RL works learn a lower bound of the ground truth Q function~\cite{mopo, combo, CQL, morel,edac,maple}, while MOHVE tries to learn the true value. BCQ~\cite{bcq}, BEAR~\cite{BEAR}, TD3+BC~\cite{td3bc} and EMaQ~\cite{emaq} are model free algorithms and constrain policy to select action similar to behavior policy. Model-based methods are also introduced for offline RL. For examples, MOPO~\cite{mopo}learns a conservative Q function by using model uncertainty; COMBO~\cite{combo} avoids uncertainty quantification by adding a Q-value minimization term alongside a standard Bellman error objective. MAPLE~\cite{maple} first introduced transfer meta-policy learning for model-based offline RL, which naturally overcame the out-of-distribution issue. In contrast with these approaches, MOHVE need not learn a conservative Q function because of its high value estimation accuracy.

Apart from the learning algorithms, some recent works turn to improve the entire training and evaluation pipeline of offline RL. NeoRL \cite{neorl}, a near real-world benchmark, points out the importance of offline policy selection. In \cite{online_evaluation}, a new evaluation paradigm is proposed that uses a limited budget to evaluate the trained policy to select the best set of hyper-parameters. This paper takes a step in this direction by evaluating the performance of the trained policy in an offline manner and choosing the best training step to stop.




\section{Reducing Value Estimation Error by Hybrid Value Estimation}\label{sec_theory}

In this section, we analyze the value estimation error of our method and the vanilla model-based method. Let us recall the procedure of value function estimation using a learned model. Let $\widehat{Q}^\pi$ be the Q function induced by the learned model $\widehat M$ and policy $\pi$. Bellman equation can be written as
\begin{equation}\label{eq_model}
\widehat{Q}^\pi(s_0,a_0)=\mathop{\mathbb{E}}\limits_{s_1\sim \widehat{T}, a_1\sim \pi} \left[\hat{r}(s_0,a_0)+\gamma\widehat{Q}^\pi(s_{1},a_{1})\right]\,.
\end{equation}

Without loss of generality, we assume that the expected TV-distance between two transition distributions is bounded by $\epsilon_m$ and the expected policy divergence is bounded by $\epsilon_\pi$, i.e., $\epsilon_m=\mathbb E_{\mathcal{D}}\mathbb E_{s,a\sim \mathcal{D}}\left[D_{\textnormal{TV}}(T(\cdot|s,a), \widehat{T}(\cdot|s,a))\right]$ and $\epsilon_\pi=\mathbb E_{\mathcal{D}}\mathbb{E}_{s\sim \mathcal{D}}\Big[D_{\textnormal{TV}}(\pi(\cdot|s), \pi_{\beta}(\cdot|s))\Big]$, where $\mathbb{E}_{\mathcal{D}}$ is the expectation with respect to the generation of dataset $\mathcal{D}$. Because the learned reward function is much more accurate than learned transition function, we ignore the error of learned reward function, and the value estimation error is bounded.

\begin{theorem}[Theorem 4.1 in~\cite{mbpo}]\label{thm_mbpo}
The error bound of the value function induced by Eq.~(\ref{eq_model}) is:
$$\mathbb E_{\mathcal{D}}\mathbb{E}_{s,a\sim\mathcal{D}}\left|Q^\pi(s,a)-\widehat{Q}^\pi(s,a)\right|\leq \frac{2\gamma R_{\max}(2\epsilon_\pi+\epsilon_m)}{(1-\gamma)^2}+\frac{4R_{\max}\epsilon_\pi}{1-\gamma}\,.$$
\end{theorem}

The theorem states that the value estimation error is related to the model learning error and policy discrepancy. Although this bound has been improved in \cite{xu.nips20} with the compounding error reduced, we consider the old result in this paper for simplicity. These two quantities both induce compounding error proportional to the effective horizon (i.e., $\frac{1}{1-\gamma}$) of the MDP. Since $\gamma$ is close to $1$, it suggests that the influence of model learning error and policy divergence is seriously magnified in value function estimation. 

This issue is not severe in the online setting as $\epsilon_m$ and $\epsilon_\pi$ can be reduced via online interaction. However, it cannot be neglected in the offline setting. Because the model can only access the dataset during the training process, it suffers from higher statistical error due to data limitation. Furthermore, since the distribution of data in the dataset may deviate from the distribution of the current policy, the policy divergence $\epsilon_\pi$ is much larger in the offline setting~\cite{bcq}.

Essentially, the value estimation error in model-based methods is due to the wrong transitions provided by the model. Note that the transitions in the offline dataset are correct, the model-induced value estimation error may be mitigated by directly utilizing the offline dataset.
Concretely, we can integrate the true transitions $\{s_t, a_t\}_{t=0}^{H+1}\sim \mathcal{D}$ into the Bellman equation, i.e., 
\begin{equation}\label{eq_correct}
\widetilde{Q}^\pi(s_0,a_0)=\mathop{\mathbb{E}}_{
    \substack{s_t, a_t\sim \mathcal{D},\\
    a_{H+1}\sim \pi(s_{H+1})}} \left[\sum_{t=0}^H\gamma^tr(s_t,a_t)+\gamma^{H+1}\widehat{Q}^\pi(s_{H+1},a_{H+1})\right]\,.
\end{equation}
This method is named Hybrid Value Estimation (HVE), which uses both offline data and the model. $\widetilde{Q}^\pi$ is the Q function learned with HVE.

For brevity, we first define $J_{\mathcal{D}}(s_0,a_0,H):=\mathbb{E}_{s_t,a_t\sim \mathcal{D}}\left[\sum_{t=0}^H\gamma^t r(s_t,a_t)\right]$ and $J_\pi(s_0,a_0,H):=\mathbb{E}_{s_t,a_t\sim d_M^\pi(s,a)}\left[\sum_{t=0}^H\gamma^t r(s_t,a_t)\right]$. Then we derive the following error bound.

\begin{theorem}
\label{thm_ValueEstimationError} 
The error bound of the value function induced by Eq.~(\ref{eq_correct}) is:
$$
\begin{aligned}
\mathbb E_{\mathcal{D}}\mathbb{E}_{s,a\sim\mathcal{D}}\left|Q^\pi(s,a)-\widetilde{Q}^\pi(s,a)\right|&\leq \sqrt{\mathbb V(J_{\mathcal{D}}(s, a, H))}+f(H) R_{\max}\epsilon_\pi+\gamma^{H+1}\epsilon_{M,\widehat{M}}\,,
\end{aligned}
$$
where $f(H)=\left(\frac{1-\gamma^{H+1}}{(1-\gamma)^2}-\frac{(H+1)\gamma^{H+1}}{1-\gamma}\right)$ and $\epsilon_{M, \widehat{M}}=\bigg(\frac{2\gamma R_{\max}(2\epsilon_\pi+\epsilon_m)}{(1-\gamma)^2}+\frac{4R_{\max}\epsilon_\pi}{1-\gamma}\bigg)$.
\end{theorem}

The variance term $\mathbb V(J_{\mathcal{D}}(s,a, H))$ in Thm.~\ref{thm_ValueEstimationError} is further quantified by the following theorem. 

\begin{theorem}\label{thm_var}
If there are $n$ trajectories in dataset $\mathcal{D}$ starting from $(s,a)$, then
$$\mathbb V(J_{\mathcal{D}}(s,a, H))=\frac{1}{n}\left(\textnormal{Var}(\text{rew})+\textnormal{Var}(\text{trans}))\right)\,,$$
\end{theorem}
where 
$\left\{\begin{aligned}
    \textnormal{Var}(rew)&=\Sigma_{t=0}^H\mathbb{E}\left[\gamma^{2t}\mathbb{V}_{t+1}[r(s_t, a_t)|s_t]|s_0=s, a_0=a\right] \\
    \textnormal{Var}(trans)&=\Sigma_{t=1}^H\mathbb E\left[\gamma^{2t}\mathbb{V}_{t}(J_\pi(s_t,a_t, H-t))|s_0=s, a_0=a\right]
\end{aligned}\right.$ and $\mathbb{V}_{t}$ is the $t$-step value defined as $\mathbb{V}[\cdot|s_0, a_0, s_1, \dots, s_{t-1},a_{t-1}]$.

\textbf{Remark 1.} $H$ can be seen as the parameter that trades off between the first two terms and the last term in the bound in Thm.~\ref{thm_ValueEstimationError}. When $H=-1$, we only use the data generated by the model. The first two terms are equal to zero, and the bound degenerates to that of Thm.~\ref{thm_mbpo}. When $H\to\infty$, the last term vanishes and the first two terms are dominating. It implies the balance is usually obtained when $H$ is neither too small nor too large.

\textbf{Remark 2.} To better understand the bound in Thm.~\ref{thm_ValueEstimationError} and Thm.~\ref{thm_mbpo}, we give an example to illustrate the improvement. Consider a MDP with deterministic reward function, and let $R_{\max}=1$ and $\gamma=0.9$. Assume policy $\pi$ and model $\hat{M}$ satisfy $\epsilon_\pi=0.1$ and $\epsilon_m=0.05$. For $H=5$, we can calculate $\mathbb V(J_{\mathcal{D}}(s, a, H))$ using Thm.~\ref{thm_var}. If $n=9$, we have $\mathbb V(J_{\mathcal{D}}(s, a, 5))\leq 1.723$. Then the bound in Thm~\ref{thm_ValueEstimationError} equals 4.55. While the bound in Thm.~\ref{thm_mbpo} equals 5.8. 

We hope to further reduce value estimation error based on Thm.~\ref{thm_ValueEstimationError}. Considering importance sampling~\cite{eligibility_traces} can adjust the value estimation of one policy into that of another policy, we use it to modify the reward provided by the offline dataset. Then we can avoid the influence of policy divergence. With importance sampling, Eq.~(\ref{eq_correct}) is reformulated as 
\begin{align}\label{eq_is}
\widetilde{Q}^\pi(s_0,a_0)=\mathop{\mathbb{E}}_{
    \substack{s_t, a_t\sim \mathcal{D},\\
    a_{H+1}\sim \pi(s_{H+1})}} \left[\sum_{t=0}^H\rho_{1:t}\gamma^tr(s_t,a_t)+\gamma^{H+1}\widehat{Q}^\pi(s_{H+1},a_{H+1}))\right]\,,
\end{align}
where $\rho_{1:t}=\rho_1\rho_2\cdots\rho_t$ for $t>0$, $\rho_{1:0}=1$ and $\rho_t=\frac{\pi(a_t|s_t)}{\pi_\beta(a_t|s_t)}$.

In this way, the second term $f(H)R_{\max}\epsilon_\pi$ equals zero. However, the first term will increase and become $\mathbb{V}(J^{\textnormal{IS}}_{\mathcal{D}}(s,a,H))$, where $J^{\textnormal{IS}}_{\mathcal{D}}(s,a,H):=\mathbb{E}_{s_t,a_t\sim \mathcal{D}}\left[\sum_{t=0}^H\rho_{1:t}\gamma^t r(s_t,a_t)\right]$. The following theorem shows that the variance will be amplified by $\rho_{1:t}$.
\begin{theorem}\label{thm_is_var}
Under the conditions of Thm.~\ref{thm_var}, the variance of Eq.~(\ref{eq_is}) is 
$$\mathbb V(J^{\textnormal{IS}}_{\mathcal{D}}(s,a, H))=\frac{1}{n}\left(\textnormal{Var}_{\textnormal{IS}}(\text{rew})+\textnormal{Var}_{\textnormal{IS}}(\text{trans}))\right)\,,$$

\text{where} 
$\left\{\begin{aligned}
    \textnormal{Var}_{\textnormal{IS}}(\text{rew})&=\Sigma_{t=0}^H\mathbb{E}\left[\rho^2_{1:t}\gamma^{2t}\mathbb{V}_{t+1}[r(s_t, a_t)|s_t]|s_0=s, a_0=a\right] \\
    \textnormal{Var}_{\textnormal{IS}}(\text{trans})&=\Sigma_{t=1}^H\mathbb E\left[\rho^2_{1:t}\gamma^{2t}\mathbb{V}_{t}(J_\pi(s_{t},a_t,H-t))|s_0=s, a_0=a\right]
\end{aligned}\right.$.
\end{theorem}

Thm.~\ref{thm_is_var} shows that the variance will be large if $\rho_{1:t}$ is large, which suggests that we need to balance the first and second term in Thm~\ref{thm_ValueEstimationError} when considering $\rho_{1:t}$.

\section{Practical Implementations of HVE}\label{sec_algorithm}
Previous sections propose the general framework of HVE, while there are several remaining aspects to discuss for a practical implementation.  In this section, we discuss (1) how to obtain the behavior policy and dynamic model, (2) how to better balance bias and variance, and (3) how to choose the step length $H$ that minimizes the error.

\subsection{Behavior Policy Learning and Model Learning}

If the behavior policy $\pi_\beta(\mathbf{a|s})$ is  available, $\rho_{1:t}$ in Eq.~(\ref{eq_is}) and $\epsilon_\pi$ can be computed directly. Otherwise, we can approximate the behavior policy by maximizing the log likelihood of observed data like~\cite{Keep_doing_what_worked}: 

\begin{equation}
\label{bc}
    \theta_{b}^* = \arg\max_{\theta_b} \mathbb{E}_{(s,a)\sim\mathcal{D}_\text{env}}\Big[\log \pi_{\theta_b}(a|s)\Big]\,,
\end{equation}
where $\theta_b$ are the parameters of the behavior policy. 

The model learning process is similar to MBPO~\cite{mbpo}, which learns an ensemble of probabilistic dynamic models and uses them for short rollouts. Each model is parameterized by a neural network whose outputs represent a Gaussian distribution with diagonal covariance: $\widehat{T}_\psi(s_{t+1}|s_t,a_t) = \mathcal{N}(\mu_\psi(s_t,a_t), \Sigma_\psi(s_t,a_t))$. The model is typically learned by maximum likelihood estimation: 
\begin{equation}
    \psi^* = \arg\max_{\psi}\mathbb{E}_{(s,a,s')\sim\mathcal{D}_{\text{env}}}[\log\widehat{T}_\psi(s'|s,a)]\,.
\label{learn model}
\end{equation}
The reward function $r$, if unknown, can be jointly learned in the same way.


\subsection{Importance Ratio Clipping}

As suggested by Thm.~\ref{thm_is_var}, to better balance the influence of the first two terms, we can clip the importance ratio $\rho_{1:t}$  the range of $(1-\epsilon, 1+\epsilon)$ to ensure that it is not far away from $1$. As $\epsilon$ increases, the first term in the bound of Thm.~\ref{thm_ValueEstimationError} increases while the second term decreases. Then we can tune $\epsilon$ to get a trade-off between the two terms. We use $\mathbb V(\bar{J}^{\textnormal{IS}}_{\mathcal{D}}(s,a,H)$ to denote the variance with clipped importance ratio, $\epsilon'_\pi$ to denote the reduced policy divergence after the truncated importance sampling, the refined error bound can be represented as
\begin{equation}\label{eq_tis_bounvd}
\sqrt{\mathbb V(\bar{J}^{\textnormal{IS}}_{\mathcal{D}}(s, a, H))}+f(H) R_{\max}\epsilon'_\pi+\gamma^{H+1}\epsilon_{M,\widehat{M}}
\end{equation}
In practice, $\epsilon$ is robust across environments, and we choose $\epsilon$ equals $0.1$ in all experiments.

\subsection{Automatic Step Length Adjustment}

Sec.~\ref{sec_theory} demonstrates that the value estimation error can be reduced by properly choosing step length $H$. Since the optimal $H$ could change across environments and training stages, it is important to automatically adjust $H$ to avoid tedious hyper-parameter tuning and ensure the optimal mean-variance trade-off.

The best $H$ is the one that minimizes value estimation error, which is trivial to compute if we know the value estimation error for every $H$. However, the exact error is intractable because the true value is not known in advance. To tackle this issue, we use the error bound in Eq.~(\ref{eq_tis_bounvd}) as a surrogate of value estimation error, and the $H$ that minimizes the error bound is chosen as the parameter, i.e., 
\begin{equation}
    \label{eq_adapt_H}
    H_b = \argmin_H \sqrt{\mathbb V(\bar{J}^{\textnormal{IS}}_{\mathcal{D}}(s,a, H))}+f(H) R_{\max}\epsilon'_\pi+\gamma^{H+1}\epsilon_{M,\widehat{M}}\,.
\end{equation}

\begin{figure}\centering
\includegraphics[width=0.6\textwidth]{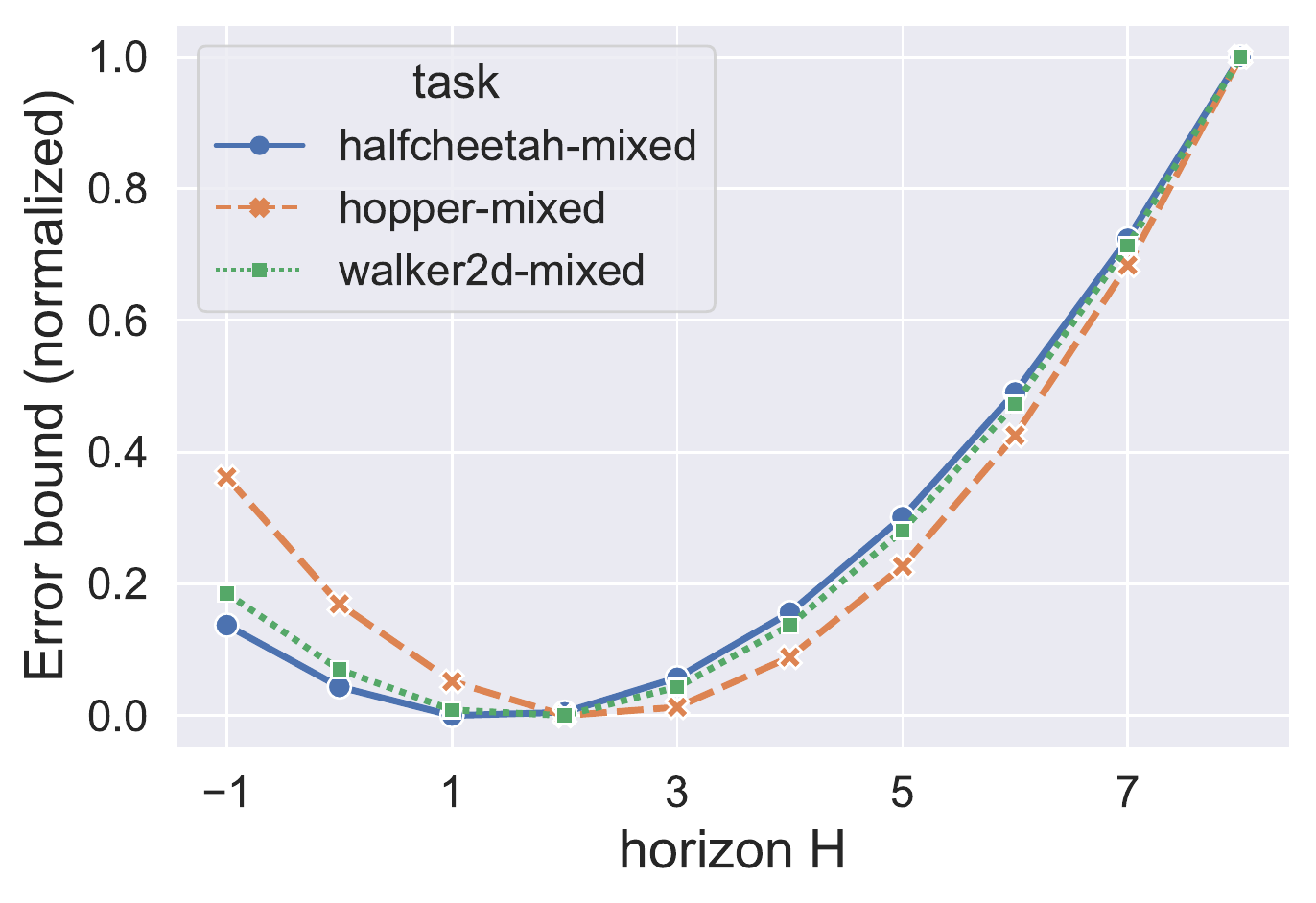}
\caption{The relationship between value estimation error bound and step length $H$. When $H$ is $-1$, the bound degenerates to vanilla model-based value estimation method.}
\label{fig:choose H}
\end{figure}

The error bound of $H$ is determined if we can compute $\mathbb V(\bar{J}^{\textnormal{IS}}_{\mathcal{D}}(s,a,H))$, $\epsilon'_\pi$ and $\epsilon_{M,\widehat{M}}$, where $\epsilon_{M,\widehat{M}}$ is constituted by $\epsilon_\pi$ and $\epsilon_m$. In the experiment, we tend to select small $H_b$ because the error of small $H_b$ is not worse than vanilla model-based value estimation while remaining unbounded if $H_b$ is too large. This implies that we tend to overestimate the first two terms but underestimate the last term. Take MuJoCo tasks as an example. $\mathbb{V}(\bar{J}^{\textnormal{IS}}_{\mathcal{D}}(s,a,H))$ is upper bounded by $(1+\epsilon)^2\sum_{t=0}^H \gamma^{2t} R_{\max}^2/4$ and $\epsilon'_\pi$ is upper bounded by $\epsilon_\pi$. We choose the upper bound as the surrogates for the reason we mentioned above. Given behavior policy $\pi_\beta$ and current policy $\pi$, $\epsilon_\pi$ can be calculated directly with the data in the dataset $\mathcal{D}$. $\epsilon_m$ is difficult to handle because the true model $M$ is not accessible. We assume the output of the true model satisfies Gaussian distribution, whose mean is provided by the transitions in the offline dataset and the variance is identical to that of the learned model. 
Fig.~\ref{fig:choose H} illustrates the relationship between the approximated error bound and $H$ in three MuJoCo tasks. The error bound decreases first and then increases as step size $H$ increases, proving that the best $H$ achieves a trade-off between three terms in Thm.~\ref{thm_ValueEstimationError}. 
Since the error bound of vanilla model-based value estimation is the value at $H=-1$, the experiment also shows value error reduction of HVE.

\section{Applications}

We present two concrete instantiations of HVE: Off-Policy HVE (OPHVE) and Model-based Offline HVE (MOHVE) which adopt HVE to improve performance in OPE and offline RL tasks. 

\textbf{OPHVE. }
OPHVE uses HVE for off-policy evaluation. According to the definition of Q function, the expected return of a policy can be expressed as
$J(\pi)=\mathbb{E}_{s\sim \rho_0, a\sim \pi}\left[Q^\pi(s,a)\right]$. With the estimated Q function, to perform OPE, we only need to sample states from initial state distribution $\rho_0$ and sample actions from the policy to be evaluated to approximate the expectation.


\textbf{MOHVE. }
MOHVE follows the actor-critic framework. The critic is modified to be updated by the HVE objective to better guide the training of the actor. 

From the error bound of Eq.~(\ref{eq_tis_bounvd}), the second term $f(H)R_{\max}\epsilon'_\pi$ and the last term $\gamma^{H+1}\epsilon_{M,\widehat{M}}$ of the error bound is proportional to policy divergence term $\epsilon_\pi$. To address this problem, we regularize the learned policy to make it close to the behavior policy, which is a common practice in previous offline RL algorithms~\cite{BEAR,BRAC,Keep_doing_what_worked}. This can be achieved by the constrained optimization formulated as follows:

\begin{equation}
\begin{aligned}
    \pi &= \argmax_{\pi} \mathbb{E}_{s\sim\mathcal{D}_{\text{mix}}, a\sim\pi(\cdot|s)} \Big[  Q^\pi(s, a)\Big]\,,
    \ \text{s.t.}\ \  \mathbb{E}_{s\sim\mathcal{D}_{\text {mix}}}\Big[D_{\text{KL}}\left(\pi(\cdot|s) || \pi_\beta(\cdot|s)\right)\Big] \leq \delta \,.
\end{aligned}
\end{equation}

The state distribution $\mathcal{D}_{\text {mix}}$ is the mixture of the offline data and the synthetic rollouts from the model: $\mathcal{D}_{\text{mix}} := \alpha \mathcal{D}_\text{env} + (1-\alpha)\mathcal{D}_\text{model}$, where $\alpha\in[0,1]$ is a hyperparameter to control the ratio of data points drawn from the offline dataset. $D_{\text{KL}}\left(\pi || \pi_\beta\right)$ is the Kullback--Leibler Divergence between the learned policy and behavior policy. We use KL divergence instead of TV divergence in practice as KL divergence has a closed-form solution for Gaussian distribution. This substitution is reasonable because of the following relationship between the two divergences: $D_{\textnormal{TV}}^2(p, q)\leq 2D_{\textnormal{KL}}(p,q)$.

To solve this constrained optimization problem, we can use the method of Lagrange multiplier to convert it to an objective that can be optimized by alternating gradient descent steps on the policy parameters $\theta$ and a penalty coefficient $\beta$:

\begin{equation}
\label{eq_pi_update}
\begin{aligned}
    \max_{\theta}\min_{\beta} \mathbb{E}_{s\sim\mathcal{D}_{\text{mix}}, a\sim\pi_{\theta}(\cdot|s)} \Big[ &Q(s, a)  + \beta (\delta - D_{\text{KL}}\left(\pi_{\theta}(\cdot|s) || \pi_\beta(\cdot|s)\right)) \Big]\,.
\end{aligned}
\end{equation}
The whole process of MOHVE is summarized in Appx.~\ref{sec_pseudo_code}.

Besides the last two terms of Eq.~(\ref{eq_tis_bounvd}), we find the constraint also reduces the variance term $\mathbb V(\bar{J}^{\textnormal{IS}}_{\mathcal{D}}(s, H))$ by limiting the scale of importance ratio as shown in the following proposition.

\begin{proposition}\label{prop_is}
If the policy $\pi$ satisfies $\mathbb{E}_{s\sim\mathcal{D}_{\text{mix}}}\left[D_{\textnormal{KL}}\left(\pi(\cdot|s) || \pi_\beta(\cdot|s)\right)\right] \leq \delta $, and suppose $c=\max_{s,a}\left\|\frac{\pi(a|s)}{\pi_\beta(a|s)}\right\|_\infty\leq \infty$, then $\mathbb{E}_{s\sim\mathcal{D}_{\text{mix}}}\mathbb{E}_{a\sim \pi_\beta(\cdot|s_t)} \rho_t^2\leq  c\delta + 1$.
\end{proposition}

\section{Experiments}

We design several experiments to demonstrate and analyze the performance of OPHVE and MOHVE. We also analyze the effect of step length $H$ and the accuracy of automatic step length adjustment. More experiment details are given in Appx.~\ref{sec_detail}.



\subsection{Off-Policy Evaluation Performance of OPHVE} 
\label{OPE experiments}

We follow the setting of DOPE~\cite{dope} to test the off-policy evaluation performance of OPHVE: For a given dataset $\mathcal{D}$, we estimate the values of a set of policies $\{\pi_1, \pi_2, ..., \pi_N\}$. The estimated values are denoted as $\{\hat{V}_1, \hat{V}_2, ..., \hat{V}_N\}$. We also deploy these policies in real environments and record their discounted returns as ground truth values, denoted as $\{V_1, V_2, ..., V_N\}$. The policies to evaluate are generated during the training process of online RL algorithms Soft Actor-Critic~\cite{sac}. DOPE is based on the D4RL benchmark of Gym-MuJoCo tasks~\cite{gym, d4rl}. The tasks are challenging for OPE methods because they have data from various sources and policies that are generated using online RL algorithms, making it less likely that the behavior and evaluation policies share similar induced state-action distributions~\cite{dope}. The performance of every method is evaluated using three different metrics: absolute error, regret@k and rank correlation. The details of the metrics are in Appx.~\ref{sec_metric}.

\begin{figure}[h]
    \centering
    \subfigure[Absolute error]{
        \includegraphics[width=0.3\textwidth]{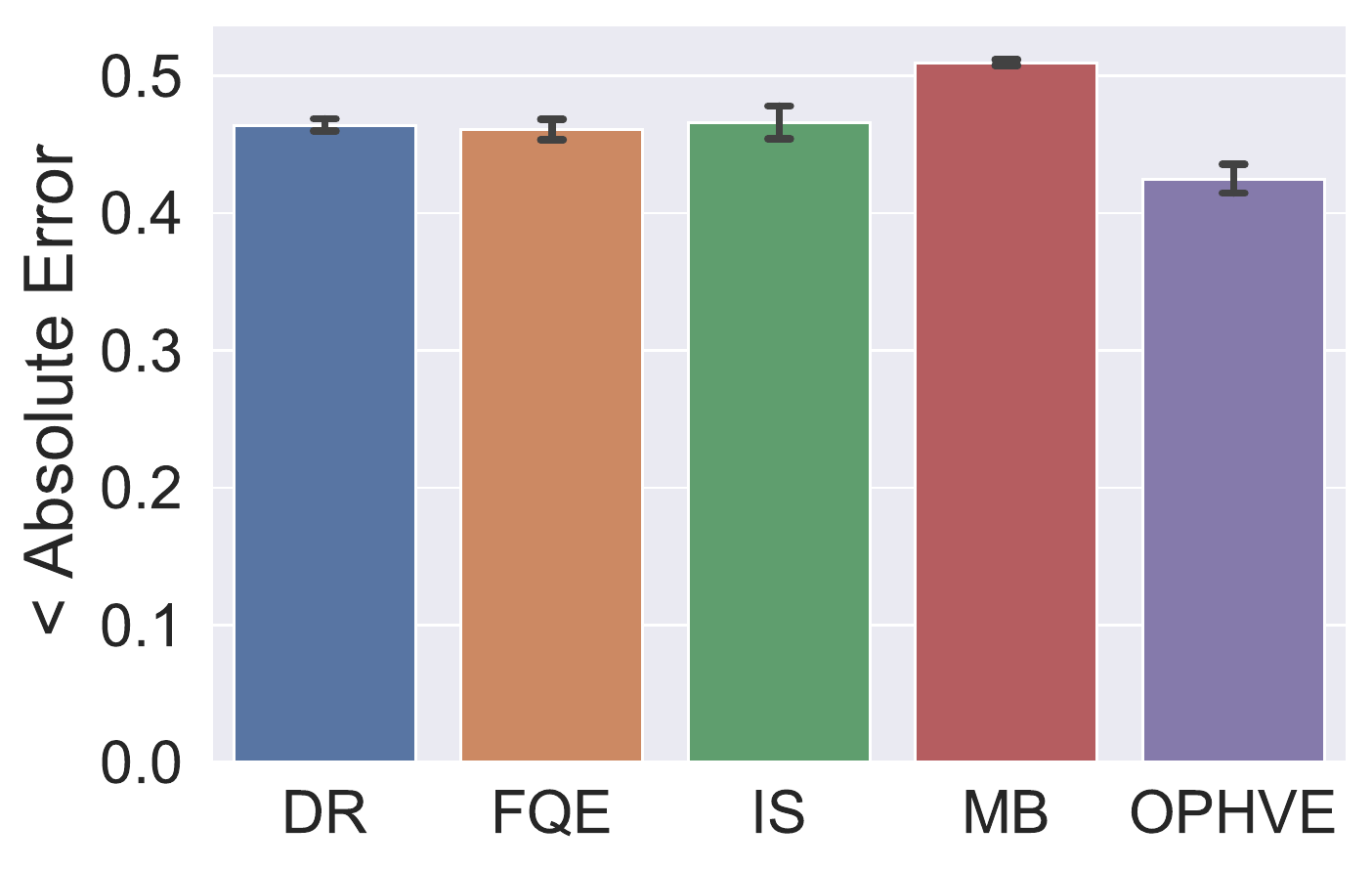}
    }
    \subfigure[Rank correlation]{
        \includegraphics[width=0.3\textwidth]{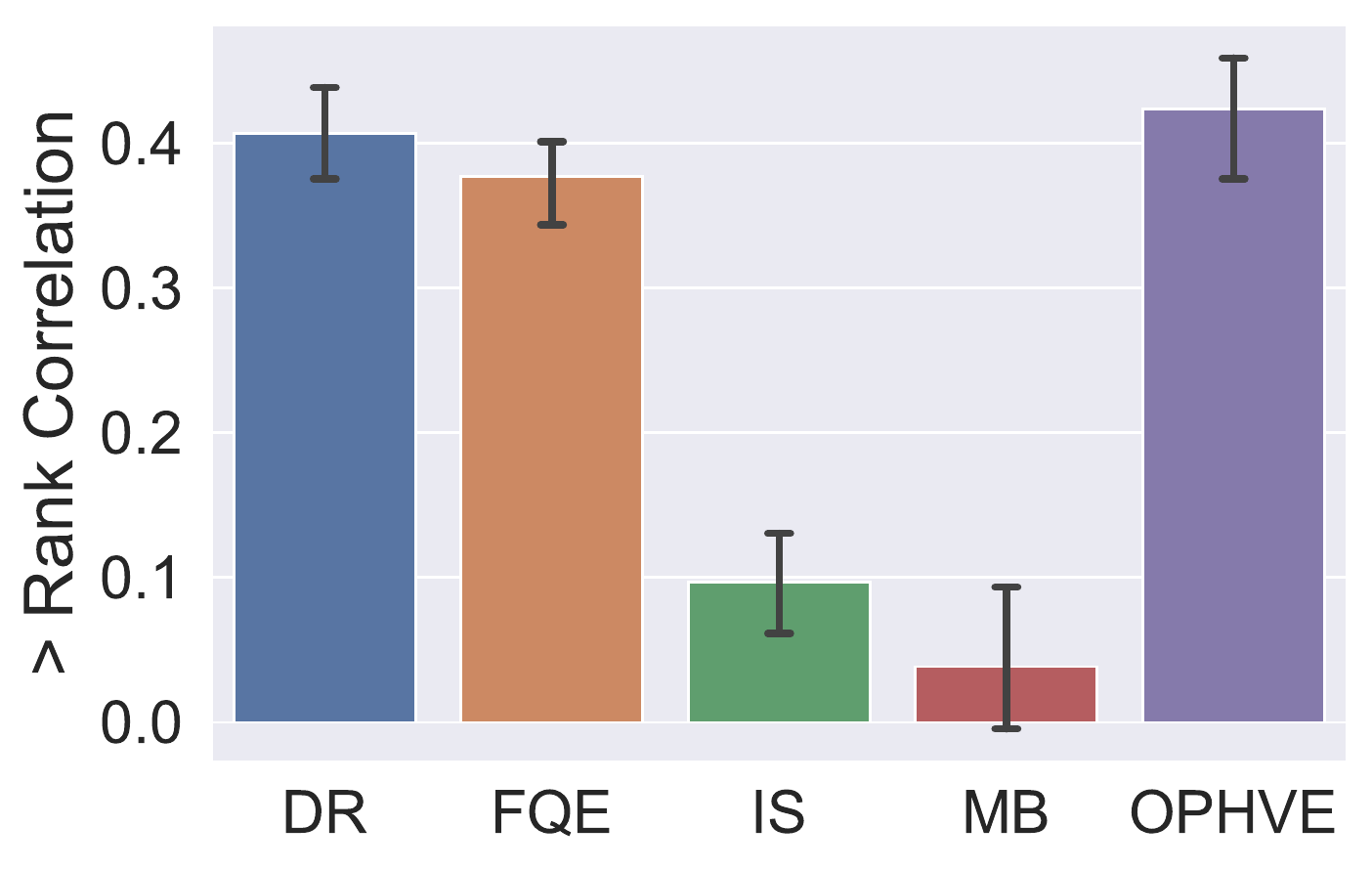}
    }
    \subfigure[Regret@1]{
        \includegraphics[width=0.3\textwidth]{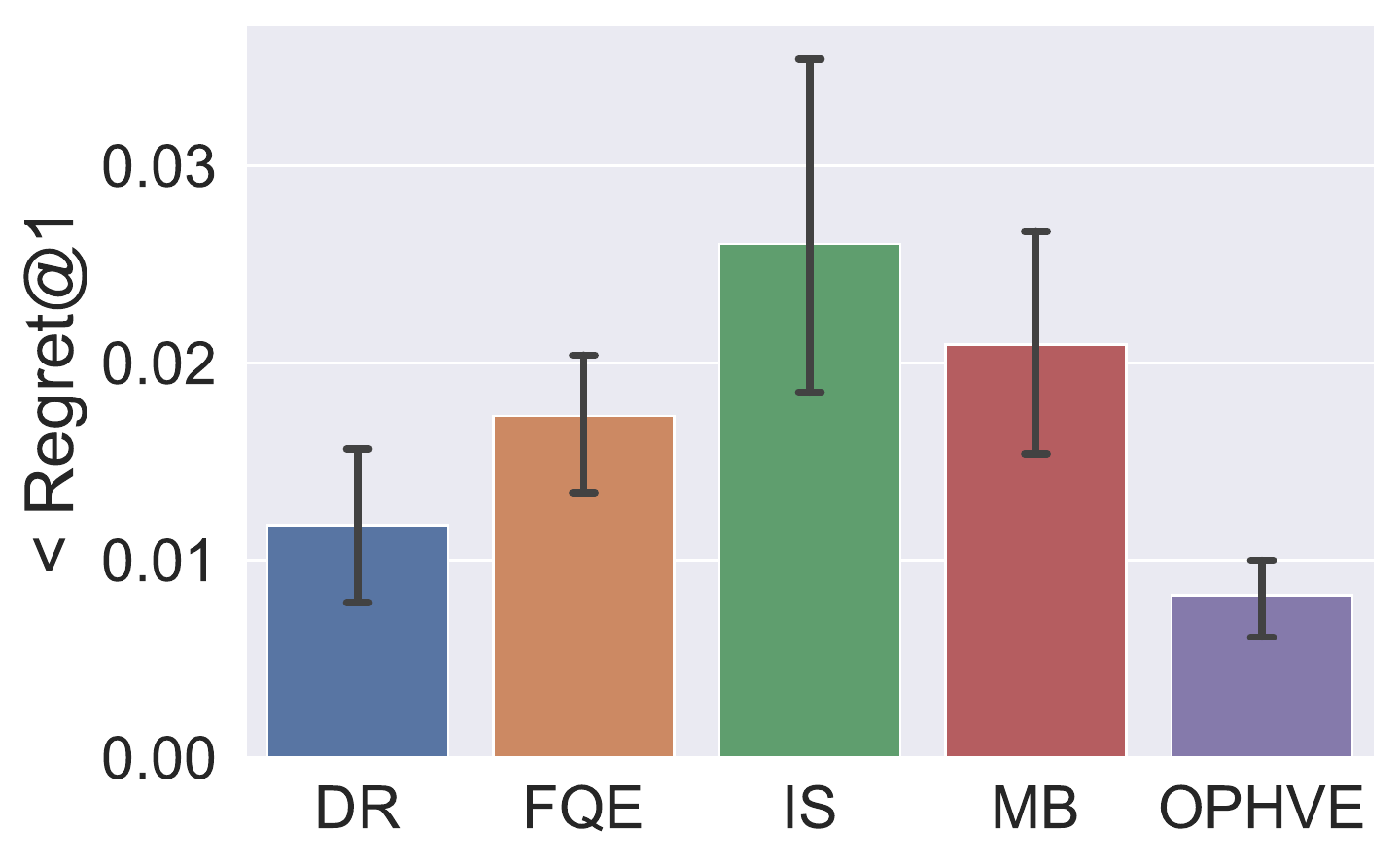}
    }
    \caption{The overall performance of off-policy evaluation in three metrics. The value of each bar is averaged over all tasks and the error bar correspond to the standard deviation. To aggregate across tasks, we normalize the policy values and evaluated policy values to range between 0 and 1.}
    \label{fig:ope}
\end{figure}

In this experiment, we compare OPHVE to several previous methods\footnote{Despite recent exciting developments in marginalized importance sampling (MIS), it does not show strong empirical performance compared to other methods~\cite{dope, ope}. Due to the difficult optimization of MIS, we defer the comparison with MIS to future work.}, including:
(1) \textbf{Fitted Q-Evaluation (FQE):} We adopt ~\cite{fqe} that estimates the policy value by iteratively performing Bellman update and use the recent implementation~\cite{statistical_bootstrapping}.
(2) \textbf{Model-Based (MB)}: MB trains dynamics and reward models by maximizing the log likelihood of the next state and reward from the offline dataset as recent successful model-based RL algorithms~\cite{learn_dynamics_model, mbpo}. Then policy evaluation is performed by rolling out in this learned model.
(3) \textbf{Importance Sampling (IS)}:  We choose the weighted importance sampling method~\cite{eligibility_traces} with the implementation of~\cite{statistical_bootstrapping}. The behavior policy is learned the same as our method.
(4) \textbf{Doubly Robust (DR)}: We perform weighted doubly-robust~\cite{doubly_roubst}, which has the least variance, and use the implementation of~\cite{statistical_bootstrapping}. 

Fig.~\ref{fig:ope} shows the performance of OPHVE and other methods in three metrics. OPHVE consistently outperforms other methods on all metrics. The results of absolute error show that hybrid value estimation reduces the value estimation error, which verifies our theory in \cref{sec_theory}. Besides, OPHVE gets a higher rank correlation and lower regret. In conclusion, OPHVE can not only perform accurate evaluation but also select the competitive policies among the polices to be evaluated. Detailed OPE results of each task are in Appx.~\ref{additional ope results}.

\subsection{Offline Reinforcement Learning Performance of MOHVE}
\label{offline rl}

We evaluate MOHVE on the D4RL benchmarks~\cite{d4rl}. 
Our baselines include four offline reinforcement learning algorithms and one imitation learning algorithm. Among the four state-of-the-art offline RL algorithms, COMBO~\cite{combo} and MOPO~\cite{mopo} are model-based methods, CQL~\cite{CQL} and TD3+BC~\cite{td3bc} are model-free method. The imitation learning algorithm is behavior cloning~\cite{BC}.

We train MOHVE for $1$ million time steps and report the average normalized score of the last ten iterations of policy update. \cref{offline experiment} shows the results of offline reinforcement learning experiments. The numbers of BC are taken from the D4RL paper, while the results for COMBO, MOPO, TD3+BC, and CQL are based on their respective papers~\cite{combo, mopo, td3bc, CQL}. In summary, MOHVE outperforms other methods in \textbf{$7$ out of the $12$ tasks}, and achieves comparable results in the rest $5$ tasks. Previous offline RL work~\cite{mopo} found that model-based methods perform worse on the task with narrow data distributions (medium, medium-expert datasets) since it's hard to learn models that generalize well. Our method performs well on these tasks because HVE reduces the value estimation error by adjusting the horizon $H$ in case that the model error is large. Besides better average performance, MOHVE generally shows a smaller variance in every task, indicating better stability and robustness. The learning curves of MOHVE are in Appx.~\ref{additional offline RL results}.

\begin{table*}[t]
\caption{\label{offline experiment} Performance of MOHVE and baselines on the Gym-MuJoCo datasets, on the normalized return metric proposed by D4RL benchmark. The score roughly ranges from 0 to 100, where 0 corresponds to a random policy performance and 100 corresponds to an expert policy performance. Each result is the average score over five random seeds $\pm$ standard deviation.}
\resizebox{1.0\textwidth}{!}{
\begin{tabular}{@{}l|l|l|l|l|l|l|l@{}}
\toprule
\textbf{Data Type} &
  \textbf{Env.} &
  \multicolumn{1}{c|}{\textbf{BC}} &
  \multicolumn{1}{c|}{\textbf{\begin{tabular}[c]{@{}c@{}}MOHVE\\ (Ours)\end{tabular}}} &
  \multicolumn{1}{c|}{\textbf{COMBO}} &
  \multicolumn{1}{c|}{\textbf{MOPO}} &
  \multicolumn{1}{c|}{\textbf{\begin{tabular}[c]{@{}c@{}}TD3\\ +BC\end{tabular}}} &
  \multicolumn{1}{c}{\textbf{CQL}} \\ \midrule
random  & halfcheetah & 2.1   & 33.39$\pm$3.23           & \textbf{38.8$\pm$3.7} & 35.4$\pm$2.5  & 10.2$\pm$1.3             & 35.4  \\
random  & hopper      & 1.6   & 12.11$\pm$0.15           & \textbf{17.9$\pm$1.4} & 11.7$\pm$0.4  & 11.0$\pm$0.1             & 10.8  \\
random  & walker2d    & 9.8   & \textbf{21.65$\pm$0.04}  & 7.0$\pm$3.6           & 13.6$\pm$2.6  & 1.4$\pm$1.6 & 7.0   \\
medium  & halfcheetah & 36.1  & \textbf{56.48$\pm$1.22}  & 54.2$\pm$1.5          & 42.3$\pm$1.6  & 42.8$\pm$0.3             & 44.4  \\
medium  & hopper      & 29.0  & \textbf{100.81$\pm$0.96} & 97.2$\pm$2.2          & 28.0$\pm$12.4 & 99.5$\pm$1.0             & 58.0  \\
medium  & walker2d    & 6.6   & 78.34$\pm$2.53           & \textbf{81.9$\pm$2.8} & 17.8$\pm$19.3 & 79.7$\pm$1.8             & 79.2  \\
mixed   & halfcheetah & 38.4  & \textbf{57.68$\pm$0.89}  & 55.1$\pm$1.0          & 53.1$\pm$2.0  & 43.3$\pm$0.5             & 46.2  \\
mixed   & hopper      & 11.8  & \textbf{93.94$\pm$2.52}  & 89.5$\pm$1.8          & 67.5$\pm$24.7 & 31.4$\pm$3.0             & 48.6  \\
mixed   & walker2d    & 11.3  & 51.87$\pm$2.25           & \textbf{56.0$\pm$8.6} & 39.0$\pm$9.6  & 25.2$\pm$5.1             & 26.7  \\
med-exp & halfcheetah & 35.8  & \textbf{103.06$\pm$2.72} & 90.0$\pm$5.6          & 63.3$\pm$38.0 & 97.9$\pm$4.4             & 62.4  \\
med-exp & hopper      & 111.9 & 111.69$\pm$0.19          & 111.1$\pm$2.9         & 23.7$\pm$6.0  & \textbf{112.2$\pm$0.2}   & 111.0 \\
med-exp & walker2d    & 6.4   & \textbf{106.79$\pm$1.97} & 103.3$\pm$5.6         & 44.6$\pm$12.9 & 105.7$\pm$2.7            & 98.7  \\ \bottomrule
\end{tabular}
}
\end{table*}

\subsection{Analysis of step length $H$}

\begin{figure}[t]
    \centering
    \includegraphics[width=\textwidth]{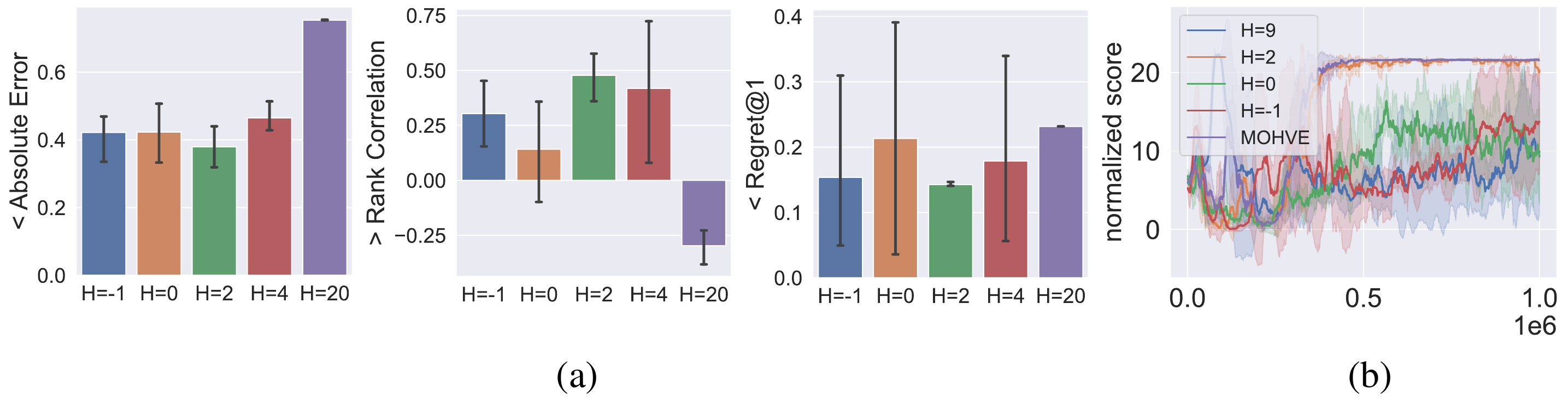}
    \caption{\label{fig:analysis} Analysis experiments results on walker2d-random. \textbf{(a)}: the OPE results in absolute error, rank correlation and regret@1. \textbf{(b)}: the offline RL learning curves with different step length $H$. Note that $H$ equals $-1$ corresponds to only using the model rollout data.}
\end{figure}

HVE benefits from the proposed way of hybridizing offline data and model rollout data. The step length $H$ plays an important role. We analyze the effect of the step length $H$ on OPHVE and MOHVE.
The results are shown in Fig.~\ref{fig:analysis}.
OPE results show that $H$ has a great impact on the performance of OPHVE.
Compared to -1, 0, 4, 20, OPHVE performs the best in all three metrics when $H$ equals $2$, which is the step length selected by the automatic adjustment. It should be noted that the variance in the experiment is induced by stochastic initialization and optimization process controlled by random seeds, while the variance term $\mathbb{V}(\bar{J}^{\textnormal{IS}}_{\mathcal{D}}(s,a,H))$ comes from the generation of the dataset, so the experiment variance cannot reflect the variance term.

Offline RL results show that our automatic adjustment of step length $H$ contributes to the performance gain of MOHVE. In detail, the vanilla model-based method (i.e., $H$ equals $-1$), and the hybrid method with $H$ equals $0$ struggle to perform well. When we increase the length of $H$ to $2$, the converged performance is improved. Finally, the performance decreases when $H$ is changed to $9$. Our method, MOHVE, converges to a higher normalized score more stably through automatically adjusting the parameter $H$ during the training process. In conclusion, the step length $H$ matters in OPE and offline RL, and our automatic selection method succeeds in finding the near-optimal one. Additional results and analysis on other tasks can be found in Appx.~\ref{sec_additional_analysis_of_H}.

\subsection{Offline Policy Selection}

\begin{figure}
\centering
\includegraphics[width=0.60\textwidth]{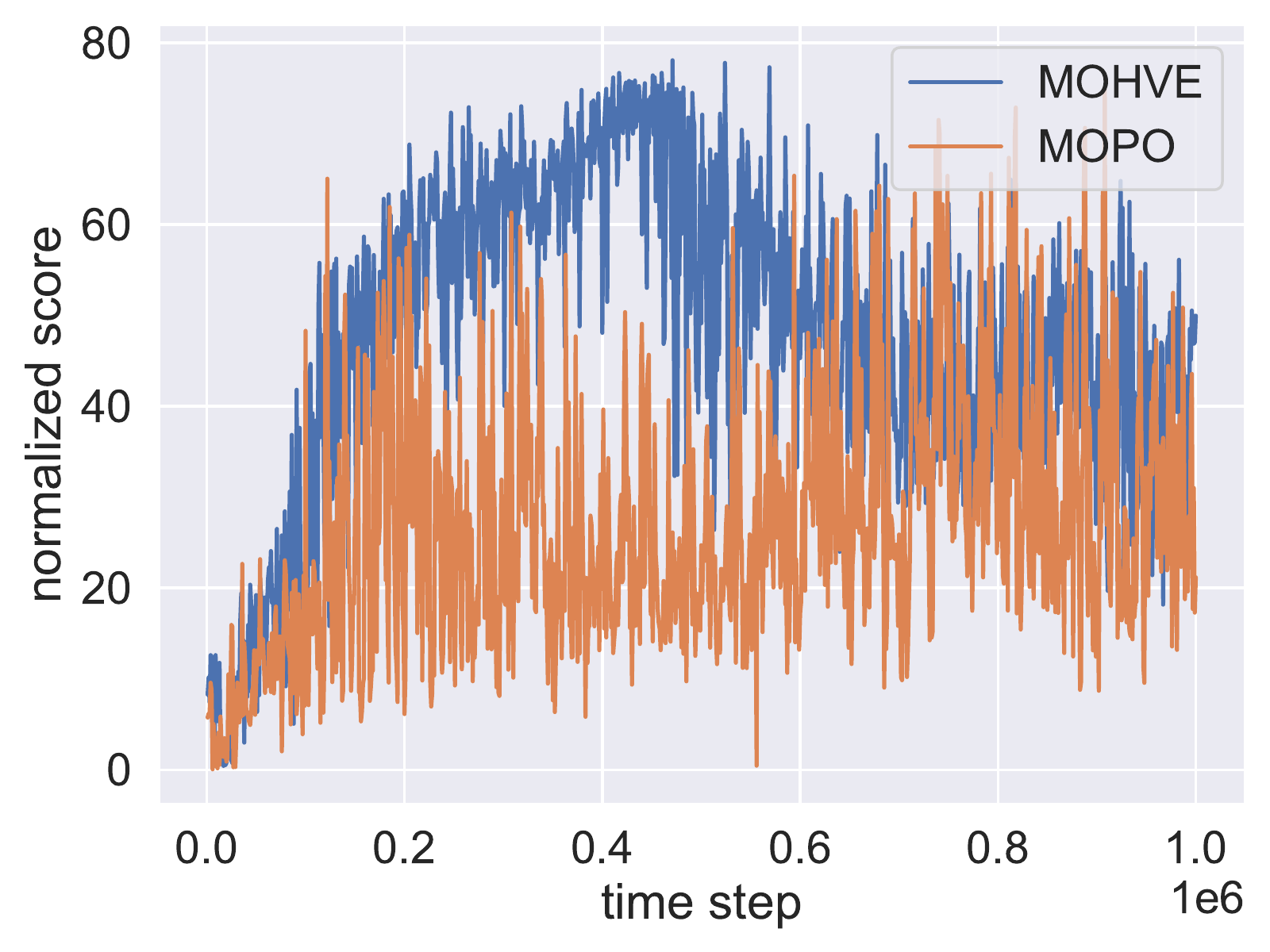}
\caption{The learning curve of MOHVE and MOPO on walker2d-mixed. There is oscillation problem in offline RL.}
\label{fig:unstable performance}
\end{figure}

 It is common that the performance of the trained policy oscillates during the training process, as illustrated in Fig.~\ref{fig:unstable performance}. Therefore, simply choosing the policy in the end may not imply a good performance. To tackle this issue, we can obtain policy snapshots during the training process and select the best of them. In real-world application, we tend to limit the budget (the amount of policies deployed online after policy selection) of online evaluation to reduce risk and cost, if not completely forbid it. Since the number of the saved snapshots is usually much larger than the budget, we need to first perform offline policy selection to  obtain a small enough policy set that includes good policies. 
 
Since OPHVE achieves a remarkable performance in OPE tasks, we integrate it to MOHVE as the offline policy selection method. Such an integration forms a completely offline training and evaluation pipeline~\cite{neorl}. We choose uniform policy selection as the baseline, since it is a strong baseline under low online evaluation budgets~\cite{online_evaluation}. We train MOHVE and save the policy every 50 epochs, and conduct offline policy selection under different online evaluation budgets. Tab.~\ref{tab:offline policy selection} shows the offline policy selection results. On one hand, OPHVE selects better policies than uniform policy selection under the same budget. On the other hand, our proposed offline RL pipeline, MOHVE+OPHVE, can further improve the performance of MOHVE, which sheds some light on designing offline RL algorithms more suitable for real word scenario.

\begin{table}[htp!]
\caption{Offline policy selection results. TOP $k$ reports the best online normalized score among $k$ policies selected by corresponding offline policy selection methods. Each result is the average score over five random seeds $\pm$ standard deviation.}
\vspace{3mm}
\label{tab:offline policy selection}
\centering
\resizebox{1.0\textwidth}{!}{
\begin{tabular}{@{}cllll|lll@{}}
\toprule
\multirow{2}{*}{\textbf{Env.}} &
  \multicolumn{1}{c}{\multirow{2}{*}{\textbf{\begin{tabular}[c]{@{}c@{}}Final\\ Iterations\end{tabular}}}} &
  \multicolumn{3}{c|}{\textbf{Uniform Policy Selection}} &
  \multicolumn{3}{c}{\textbf{OPHVE}} \\ \cmidrule(l){3-8} 
 &
  \multicolumn{1}{c}{} &
  \multicolumn{1}{c}{TOP1} &
  \multicolumn{1}{c}{TOP3} &
  \multicolumn{1}{c|}{TOP5} &
  \multicolumn{1}{c}{TOP1} &
  \multicolumn{1}{c}{TOP3} &
  \multicolumn{1}{c}{TOP5} \\ \midrule
walker2d-mixed & 51.87$\pm$2.25 & 43.09$\pm$0.81 & 55.96$\pm$1.70 & 60.05$\pm$2.69 & 56.82$\pm$7.68  & 62.22$\pm$6.00 & 66.43$\pm$5.54 \\ \midrule
hopper-mixed   & 93.94$\pm$2.52 & 66.78$\pm$8.73 & 89.31$\pm$6.30 & 94.78$\pm$3.28 & 82.45$\pm$22.30 & 97.06$\pm$2.44 & 97.06$\pm$2.44 \\ \bottomrule
\end{tabular}
}
\end{table}

\section{Conclusion}
In this paper, we propose Hybrid Value Estimation (HVE) to perform a more accurate value function estimation in the offline setting. It automatically adjusts the step length parameter to get a bias-variance trade-off. The error bound of HVE is better than that of vanilla model-based value estimation. We provide two concrete algorithms OPHVE and MOHVE and thoroughly discuss the details of the implementation. Empirical evaluations on MuJoCo tasks corroborate the theoretical claim, showing that HVE indeed selects the best step length and contributes to the performance in off-policy evaluation and offline RL tasks. We also test the performance of integrating MOHVE with OPHVE to select the best policy during training in an offline manner, which 
could shed some light on improving the training and evaluation pipeline of offline RL.

\bibliography{main}
\bibliographystyle{unsrt}
\clearpage
\newpage
\appendix
\section{Proof of results in the main paper}
\subsection{Proof of Thm.~\ref{thm_ValueEstimationError}}
To provide proof of Thm.~\ref{thm_ValueEstimationError}, we need Thm.~\ref{thm_mbpo}, Lem.~\ref{lemma_bias} and Lem.~\ref{lemma_var}.

\begin{lemma}\label{lemma_bias}
Let $J_{\pi}(s_w, a_w, H)=\mathbb E_{\pi}[\sum_{t=0}^H\gamma^tr(s_t, a_t)|s_0=s_w, a_0=a_w]$, then 
$$\mathbb E_{\mathcal{D}}\left|J_{\pi}(s_w, a_w, H)-J_{\pi_\beta}(s_w, a_w, H)\right|\leq \left(\frac{1-\gamma^{H+1}}{(1-\gamma)^2}-\frac{(H+1)\gamma^{H+1}}{1-\gamma}\right) R_{\max}\epsilon_\pi$$
\end{lemma}
\begin{proof}
Let $d_{t,w}(s)$ represents the distribution of state at time $t$ by starting with $(s_w, a_w)$ and following policy $\pi$, $d^\beta_{t,w}(s)$ represents the distribution of state at time $t$ by starting at with $(s_w,a_w)$ and following policy $\pi_\beta$, and $d'_{t,w}(s)$ represents the distribution of states at time $t$ by following $\pi_\beta$ conditioned on starting with $(s_w,a_w)$ and the fact $\pi_\beta$ performs different from $\pi$ at least one step. Then
$$d^\beta_{t,w}(s)=p_{t-1}d_{t,w}(s)+(1-p_{t-1})d'_{t,w}(s),$$
where $p_{t-1}$ is the probability $\pi_\beta$ performs the same as $\pi$ in $t-1$ steps. 

For simplification, we denote $\pi(\cdot|s)$ as $\pi(s)$, denote $\mathbb{E}_{\mathcal{D}}\mathbb{E}_{s\sim \mathcal{D}}$ as $\mathbb{E}_{s, \mathcal{D}}$. Notice that 
\begin{align*}
&\quad \ \mathbb{E}_{s, \mathcal{D}}\textnormal{Pr}(\pi(s)=\pi_\beta(s))\\
&=\mathbb{E}_{s, \mathcal{D}} \int \min(\pi(a|s),\pi_\beta(a|s))da\\
&=\mathbb{E}_{s, \mathcal{D}} \bigg(\int \pi(a|s) da - \int \big(\pi(a|s)-\min(\pi(a|s),\pi_\beta(a|s))\big)da\bigg)\\
&=\mathbb{E}_{s, \mathcal{D}} \Big(1-D_{\textnormal{TV}}(\pi(s),\pi_\beta(s))\Big)\\
&=1-\epsilon_\pi
\end{align*}
Let $p=\mathbb{E}_{s, \mathcal{D}}\textnormal{Pr}(\pi(s)=\pi_\beta(s))$, we have
\begin{align*}
&p_t\geq p_{t-1}p\geq p^t\\
\end{align*}
Then 
\begin{equation}\label{eq_ds}
\begin{aligned}
    \mathbb{E}_{s, \mathcal{D}}\left|d^\beta_{t,w}(s)-d_{t,w}(s)\right|&=\mathbb{E}_{s, \mathcal{D}}(1-p_{t-1})\left|d'_{t,w}(s)-d_{t,w}(s)\right|\\
    &\leq \mathbb{E}_{s, \mathcal{D}}(1-p_{t-1})\leq 1-p^{t}
\end{aligned}
\end{equation}
We calculate $\mathbb{E}_{\mathcal{D}}\left|J_{\pi_\beta}(s_w, a_w, H)-J_{\pi}(s_w, a_w, H)\right|$ as follows.
\begin{align*}
    &\quad \ \mathbb{E}_{\mathcal{D}}\left|J_{\pi_\beta}(s_w, a_w, H)-J_{\pi}(s_w, a_w, H)\right|\\
    &= \mathbb{E}_{\mathcal{D}}\left|\mathbb E_{s_t\sim d^\beta_t, a_t\sim \pi_\beta s'_t\sim d_t,a'_t\sim \pi}\sum_{t=0}^H\gamma^t \left(r(s_t, a_t)-r(s'_t, a'_t)\right)\right |\\
    &= \left|\sum_{t=0}^H \gamma^t \int (d^\beta_{t,w}(s)\pi_\beta(a|s)-d_{t,w}(s)\pi(a|s))r(s,a) dsda\right|\\
    &\leq \sum_{t=0}^H \gamma^t \int \left|d^\beta_{t,w}(s)\pi_\beta(a|s)-d_{t,w}(s)\pi(a|s)\right|r(s,a)dsda\\
\end{align*}

Using triangle inequality, 
\begin{align*}
    &\quad \ \sum_{t=0}^H \gamma^t \int \left|d^\beta_{t,w}(s)\pi_\beta(a|s)-d_{t,w}(s)\pi(a|s)\right|r(s,a)dsda\\
    &\leq \sum_{t=0}^H \gamma^t \int \bigg( \left|d^\beta_{t,w}(s)-d_{t,w}(s)\right| \pi(a|s) +\left|\pi(a|s)-\pi_\beta(a|s)\right|d^\beta_{t,s_w}(s)\bigg)r(s,a)dsda\\
    &\leq \sum_{t=0}^H \gamma^t R_{\max}\int \bigg( \left|d^\beta_{t,s_w}(s)-d_{t,s_w}(s)\right| +\left|\pi(a|s)-\pi_\beta(a|s)\right|\bigg)dsda\\
\end{align*}
Using Eq.~(\ref{eq_ds})
\begin{align*}
    &\quad \ \mathbb{E}_{\mathcal{D}}\left|J_{\pi_\beta}(s_w, a_w, H)-J_{\pi}(s_w, a_w, H)\right|\\
    &\leq R_{\max} \sum_{t=0}^H\gamma^t \Big((1-p^{t})+\epsilon_\pi\Big)\\
    &\leq R_{\max}\sum_{t=0}^H\gamma^t \Big(1-(1-\epsilon_\pi)^{t}+\epsilon_\pi\Big)\\
    &\leq R_{\max} \sum_{t=0}^H \gamma^t t\epsilon_\pi\\
    &=\left(\frac{1-\gamma^{H+1}}{(1-\gamma)^2}-\frac{(H+1)\gamma^{H+1}}{1-\gamma}\right) R_{\max}\epsilon_\pi
\end{align*}
And then we get 
$$\mathbb E_{\mathcal{D}}\left|J_{\pi}(s_w,a_w, H)-J_{\pi_\beta}(s_w, a_w, H)\right|\leq \left(\frac{1-\gamma^{H+1}}{(1-\gamma)^2}-\frac{(H+1)\gamma^{H+1}}{1-\gamma}\right) R_{\max}\epsilon_\pi$$
\end{proof}

\begin{lemma}\label{lemma_var}
Let $J_{\mathcal{D}}(s_w, a_w,  H)=\mathbb E_{s_t,a_t\sim \mathcal{D}}[\sum_{t=0}^H\gamma^tr(s_t, a_t)|s_0=s_w]$, we have
$$\mathbb E_{\mathcal{D}}\left|J_{\pi_\beta}(s_w,a_w, H)-J_{\mathcal{D}}(s_w, a_w, H)\right|\leq \sqrt{\mathbb V(J_{\mathcal{D}}(s_w, a_w, H))}$$
\end{lemma}
\begin{proof}
Note that $\mathbb E_{\mathcal{D}}J_{\mathcal{D}}(s_w,a_w,  H)=J_{\pi_\beta}(s_w, a_w, H)$. Using the bias-variance decomposition of MSE, we have
\begin{align*}
    &\quad \ \mathbb E_{\mathcal{D}}\left|J_{\pi_\beta}(s_w,a_w, H)-J_{\mathcal{D}}(s_w, a_w, H)\right|^2\\
    &=(J_{\pi_\beta}(s_w,a_w,  H)-\mathbb E_{\mathcal{D}}J_{\mathcal{D}}(s_w, a_w, H))^2+\mathbb V(J_{\mathcal{D}}(s_w, a_w, H))\\
    &=\mathbb V(J_{\mathcal{D}}(s_w, a_w, H))
\end{align*}
Using Jensen's inequality, 
$$\left(\mathbb E_{\mathcal{D}}\left|J_{\pi_\beta}(s_w,a_w,  H)-J_{\mathcal{D}}(s_w, a_w, H)\right|\right)^2\leq  \mathbb E_{\mathcal{D}}\left|J_{\pi_\beta}(s_w, a_w, )-J_{\mathcal{D}}(s_w, a_w, H)\right|^2$$
Then we conclude the proof.
\end{proof}

\begin{proof}[Proof of Thm.~\ref{thm_ValueEstimationError}]
Note that
$$Q^\pi(s,a)=J_\pi(s, a, H)+\gamma^{H+1}\mathbb E_{s'\sim P^\pi_{H+1} (s, a), a'\sim \pi}Q^\pi(s', a')\,,$$
$$\widehat{Q}^\pi(s,a)=J_{\mathcal{D}}(s,a, H)+\gamma^{H+1}\mathbb E_{s'\sim P^\pi_{H+1} (s,a), a'\sim \pi}Q^\pi_{\widehat M}(s', a')\,,$$
where $P^\pi_{H+1}(s,a)$ denotes the transition starting from $(s,a)$ and following policy $\pi$ for $H$ steps.

\begin{align*}
    &\quad \ \mathbb E_{\mathcal{D}}\left|Q^\pi(s,a)-\widehat{Q}^\pi(s,a)\right|\\
    &\overset{(a)}{\leq} \mathbb E_{\mathcal{D}}\left|J_\pi(s,a, H)-J_{\mathcal{D}}(s, a, H)\right|+\mathbb E_{\mathcal{D}}\gamma^{H+1}\left|\mathbb E_{s'\sim P^\pi_{H+1} (s,a)}Q^\pi(s')-\mathbb E_{s'\sim P^\pi_{H+1} (s,a)}Q^\pi_M(s')\right|\\
    &\overset{(b)}{\leq} \sqrt{\mathbb V(J_{\mathcal{D}}(s,a, H))}+\left(\frac{1-\gamma^{H+1}}{(1-\gamma)^2}-\frac{(H+1)\gamma^{H+1}}{1-\gamma}\right) R_{\max}\epsilon_\pi\\
    &\qquad \ +\gamma^{H+1}\bigg(\frac{2\gamma R_{\max}(2\epsilon_\pi+\epsilon_m)}{(1-\gamma)^2}+\frac{4R_{\max}\epsilon_\pi}{1-\gamma}\bigg)\,,
\end{align*}
where (a) uses triangle inequality and (b) uses Lem.~\ref{lemma_bias} and Lem.~\ref{lemma_var}.

\end{proof}

\subsection{Proof of Thm.~\ref{thm_var} and Thm.~\ref{thm_is_var}}

Thm.~\ref{thm_is_var} can be reduced to Thm.~\ref{thm_var} if $\rho_{0:t}=1$, so we only need to prove Thm.~\ref{thm_is_var}.

Let $\hat{J}(s_w,a_w, H)=\sum_{t=0}^H\gamma^t \rho_{0:t}r(s_t,a_t)$ given $s_0=s_w,a_0=a_w$, then
$$\hat{J}(s_t, a_t, H-t)=\left (r(s_t, a_t)+\gamma \rho_{t+1}\hat{J}(s_{t+1},a_{t+1}, H-t-1)\right)$$ 
We use the shorthand: $\mathbb{E}_t[\cdot]:=\mathbb E[\cdot|s_0,a_0,\dots,s_{t-1},a_{t-1}]$ for conditional expectations, and $\mathbb{V}_t[\cdot]$ for variances similarly.

\begin{lemma}\label{lemma_recursive}
$$
\begin{aligned}
\mathbb{V}_t(\hat{J}(s_t,a_t,H-t))&=\mathbb{E}_t\left[\rho_t^2\mathbb{V}_{t+1}\left[r(s_t, a_t) \right]\right]+\mathbb{E}_t\left[\rho_t^2\gamma^2\mathbb{V}_{t+1}[\hat{J}(s_{t+1},a_{t+1},H-t-1)\right]\\
&\quad \ +\mathbb{V}_t(J_{\pi}(s_t, \pi(s_t),H-t))
\end{aligned}
$$
$$\mathbb{V}_{t+1}\left[\hat{J}(s_{t},a_t, H)\right]=\mathbb{E}_{t+1}\left[\rho_t^2\mathbb{V}_{t+1}\left[r(s_t, a_t)\right]\right]+\mathbb{E}_{t+1}\left[\rho_t^2\gamma^2\mathbb{V}_{t+1}[\hat{J}_{\mathcal{D}}(s_{t+1},a_{t+1},H-t-1)\right]$$
\end{lemma}
\begin{proof}
\begin{align*}
    &\quad \ \mathbb{V}_t(\hat{J}(s_t,a_t,H-t))\\
    &=\mathbb{E}_t\left[\hat{J}^2(s_t,a_t,H-t)\right]-\big(\mathbb{E}_t J_{\pi}(s_t, \pi_\beta(s_t),H-t)\big)^2\\
    &=\mathbb{E}_t\left[\left(r(s_t,a_t)+\gamma\rho_{t+1}\hat{J}(s_{t+1}, a_{t+1},H-t-1)\right)^2-J^2_\pi(s_t,a_t,H-t)\right]+\mathbb{V}_t(J_{\pi}(s_t, a_t,H-t))\\
    &=\mathbb{E}_t\left[\left(J_{\pi}(s_t,a_t, H-t)+r(s_t,a_t)+\gamma\rho_{t+1}\hat{J}(s_{t+1},a_{t+1}, H-t-1)-J_{\pi}(s_t, a_t, H-t)\right)^2\right.\\
    &\left. \qquad \qquad -J^2_\pi(s_t,a_t,H-t)\right]+\mathbb{V}_t(J_{\pi}(s_t,a_t, H-t))\\
    &=\mathbb{E}_t\left[\left(J_{\pi}(s_t,a_t, H-t)+\Big(r(s_t,a_t)-R(s_t,a_t)+\gamma\rho_{t+1}\Big(\hat{J}(s_{t+1}, a_{t+1}, H-t-1) \right.\right.\\
    &\left.\left.  \qquad \qquad -\mathbb{E}_{t+1}[J_{\pi}(s_{t+1},a_{t+1}, H-t-1)]\Big)\Big)\right)^2-J^2_\pi(s_t,a_t, H-t)\right]\\
    &\qquad +\mathbb{V}_t(J_{\pi}(s_t, a_t, H-t))\\
    &\overset{(a)}{=}\mathbb{E}_t\left[\mathbb E\left[J^2_{\pi}(s_t, a_t, H-t)-J^2_\pi(s_t,\pi(s_t), H-t)|s_t\right]\right]+\mathbb{E}_t\left[\mathbb E\left[(r(s_t,a_t)-R(s_t,a_t))^2 |s_t,a_t\right]\right]\\
    & \qquad +\mathbb{E}_t\left[\mathbb E\left[\gamma^2\rho_{t+1}^2(\hat{J}(s_{t+1},a_{t+1}, H-t-1)-\mathbb{E}_{t+1}J_{\pi}(s_{t+1},a_{t+1}, H-t-1))^2|s_t,a_t\right]\right]\\
    &\qquad +\mathbb{V}_t(J_{\pi}(s_t,a_t, H-t))\\
    &=\mathbb{E}_t\left[\mathbb{V}_{t+1}\left[r(s_t, a_t) \right]\right]+\mathbb{E}_t\left[\rho_{t+1}^2\gamma^2\mathbb{V}_{t+1}[\hat{J}(s_{t+1},a_{t+1},H-t-1)\right]+\mathbb{V}_t(J_{\pi}(s_t, a_t,H-t))
\end{align*}
where $R(s,a)=\mathbb E[r(s,a)]$ and (a) uses the fact that conditioned on $s_t$ and $a_t$, $r(s_t,a_t)-R(s_t,a_t)$ and $\hat{J}(s_{t+1},a_{t+1}, H-t-1)-\mathbb{E}_{t+1}J_{\pi}(s_{t+1}, a_{t+1},H-t-1)$ are independent and have zero means, and all other terms are constants.

Similarly, 

\begin{align*}
    &\quad \ \mathbb{V}_{t+1}\left[\hat{J}(s_{t},a_t, H)\right]\\
    &=\mathbb{E}_{t+1}\left[\hat{J}^2(s_t,a_t, H-t)\right]-\left(\mathbb{E}_{t+1}\left[J_{\pi}(s_t,a_t, H-t)\right]\right)^2\\
    &=\mathbb{E}_{t+1}\left[\mathbb{V}_{t+1}\left[r(s_t, a_t)\right]\right]+\mathbb{E}_{t+1}\left[\rho_{t+1}^2\gamma^2\mathbb{V}_{t+1}[\hat{J}_{\mathcal{D}}(s_{t+1},a_{t+1},H-t-1)\right]\\
    &\quad \ +\mathbb{V}_{t+1}[J_{\pi}(s_t, a_t,H-t)]\\
    &=\mathbb{E}_{t+1}\left[\mathbb{V}_{t+1}\left[r(s_t, a_t)\right]\right]+\mathbb{E}_{t+1}\left[\rho_{t+1}^2\gamma^2\mathbb{V}_{t+1}[\hat{J}_{\mathcal{D}}(s_{t+1},a_{t+1},H-t-1)\right]\,.\\
\end{align*}
The last equality follows the fact that $\mathbb{V}_{t+1}\left[J_\pi(s_t,a_t,H-t)\right]=0$.
\end{proof}

\begin{proof}[Proof of Thm.~\ref{thm_is_var}]
\begin{align*}
    &\quad \ \mathbb V(\hat{J}(s_w,a_w, H))\\
    &=\mathbb{V}[\hat{J}(s_0,a_0,H)|s_0=s_w,a_0=a_w]\\
    &=\mathbb{V}_1[\hat{J}(s_0,a_0,H)]\\
    &\overset{(a)}{=}\mathbb{E}_1\left[\mathbb{V}_1\left[r(s_0, a_0)\right]\right]+\mathbb{E}_1\left[\rho_1^2\gamma^2\mathbb{V}_{1}[\hat{J}(s_{1},a_1,H-1)]\right]\\
    &\overset{(b)}{=}\sum_{t=0}^H\mathbb{E}_1\left[\rho^2_{1:t}\gamma^{2t}\mathbb{V}_{t+1}[r(s_t, a_t)]\right]+\sum_{t=1}^H\mathbb E_1\left[\rho^2_{1:t}\gamma^{2t}\mathbb{V}_{t+1}(J_\pi(s_{t+1},a_{t+1},H-t))\right]\\
    &=\sum_{t=0}^H\mathbb{E}\left[\rho^2_{1:t}\gamma^{2t}\mathbb{V}_{t+1}[r(s_t, a_t)]|s_0=s_w,a_0=a_w\right]\\
    &\quad \ +\sum_{t=1}^H\mathbb E\left[\rho^2_{1:t}\gamma^{2t}\mathbb{V}_{t+1}(J_\pi(s_{t+1},a_{t+1},H-t))|s_0=s_w,a_0=a_w\right]\,,
\end{align*}
where we define $\rho_{1:0}=1$, (a) and (b) use Lem.~\ref{lemma_recursive}.

According to the definition of $J_{\mathcal{D}}(s_w,a_w,H)$, $\mathbb{V}(J_{\mathcal{D}}(s_w,a_w,H))=\mathbb V(\hat{J}(s_w,a_w, H))/n$, where $n$ is the number of trajectories in dataset $\mathcal{D}$ starting from $(s_w,a_w)$.
\end{proof}

\subsection{Proof of Prop.~\ref{prop_is}}

We first introduce the concept of the chi-square divergence. 
\begin{definition}
For r.v. P and Q, the ch-square divergence of P and Q, denoted as $\chi^2(P,Q)$, is defined as
$$\chi^2(P,Q)=\int\frac{dP^2}{dQ}-1$$
\end{definition}

The following lemma reveal the relationship between chi-square divergence and KL divergence.
\begin{lemma}\label{lemma_divergence}
Let $c=\left\|\frac{dP}{dQ}\right\|_\infty\leq \infty$, then
$$cD_{\textnormal{KL}}(P,Q)\geq \chi^2(P,Q)$$
\end{lemma}
\begin{proof}
Note that $\log x\geq \frac{x-1}{x}$,
$$cD_{\textnormal{KL}}(P,Q)\geq \left\|\frac{dP}{dQ}\right\|_\infty\int dP\left(\frac{dP/dQ-1}{dP/dQ}\right)\geq \chi^2(P,Q)$$
\end{proof}

\begin{proof}[Proof of Prop.~\ref{prop_is}]
Let P be $\pi(\cdot|s)$ and Q be $\pi_\beta(\cdot|s)$, Lem.~\ref{lemma_divergence} can be rewritten as
$$cD_{\textnormal{KL}}(\pi(\cdot|s), \pi_\beta(\cdot|s))\geq \chi^2(\pi(\cdot|s), \pi_\beta(\cdot|s))$$
Note that 
\begin{align*}
\chi^2(\pi(\cdot|s), \pi_\beta(\cdot|s))&=\int\left(\frac{\pi^2(a|s)}{\pi_\beta(a|s)}-1\right)da\\
&=\int\pi_\beta(a|s)\left(\frac{\pi^2(a|s)}{\pi^2_\beta(a|s)}\right)da-1\\
&=\mathbb{E}_{a\sim \pi_\beta}\left[\frac{\pi^2(a|s)}{\pi^2_\beta(a|s)}\right]-1\,,
\end{align*}
we can get the relationship between importance ratio and policy KL divergence as
$$cD_{\textnormal{KL}}(P,Q)\geq \mathbb{E}_{a\sim \pi_\beta}\left[\rho_t^2\right]-1\,.$$

The final result can be obtained by applying the expectation of state distribution to both sides of the inequality.
\end{proof}


\section{Experiment Details}
\label{sec_detail}

\subsection{Detailed Training Procedure of MOHVE}
\label{sec_pseudo_code}
According to the theory of Sec.~\ref{sec_theory}, we need to build two separate networks to represent $\widehat{Q}$ and $\widetilde{Q}$. In practice, simply using one neural network is sufficient for good performance. Concretely, let $Q_{\phi_k}$ be the Q function in the $k$-th iteration, we update Q function using Eq.~(\ref{eqn:Q_update}) and (\ref{eq_model_update}) iteratively. 
\begin{align}
\label{eqn:Q_update}
    &\phi_{k+1} = \phi_k - \frac{\eta}{2}\nabla_\phi \mathbb{E}_{s,a\sim \mathcal{D}_\text{env}}
    \left[\left(Q(s,a) - \text{TIS}^H(s,a)\right)^2\right]\,,\\
    \label{eq_model_update}
     &\phi_{k+2}=\phi_{k+1}-\frac{\eta}{2}\nabla_\phi \mathbb{E}_{s,a\sim \mathcal{D}_\text{model}}
    \left[\left(Q(s,a) - \mathcal{\hat{B}}^{\pi} Q^{k+1}(s,a)\right)^2\right]\,, 
\end{align}
where $\eta$ is the learning rate and $\text{TIS}^H(s_0,a_0) := \sum_{t=0}^{H} \gamma^{t} \rho_{0:t} r(s_t,a_t) + \gamma^{H+1}  Q^k(s_{H+1}, a_{H+1})$. 

We summarize the whole training process of MOHVE in Alg.~\ref{algo: MOHVE}. 

\begin{algorithm}[htb]
    \caption{MOHVE: Model-based Offline RL with Hybrid Value Estimation}
    \label{algo: MOHVE}
\begin{algorithmic}[1]
\REQUIRE offline data $\mathcal{D}_{\text{env}}$, critic network $Q_\phi$ with parameter $\phi$, policy network $\pi_\theta$ with parameter $\theta$, model rollout horizon $h$, real data ratio $\alpha$, KL regularization parameter $\delta$, $N_H$ epochs between $H$ update.
\STATE Train an ensemble of $N$ probabilistic dynamics $\{\hat{T}^i_\psi(s',r|s,a)=\mathcal{N}(\mu^i_\psi(s,a),\Sigma^i_\psi(s,a))\}_{i=1}^N$ on $\mathcal{D}_{\text{env}}$ with Eq.~(\ref{learn model}).
\STATE Learn behavior policy $\pi_\beta$ with Eq.~(\ref{bc}).
\STATE Initialize policy $\pi$ and empty replay buffer $\mathcal{D}_{\text{model}}\leftarrow \emptyset$.
\FOR{epoch $e = 1,2,...,$}
\STATE Sample $b$ states $s_1$ from $\mathcal{D}_{\text{env}}$ as the initial states of model rollouts.
\STATE Rollout $h$ step in dynamics $\{\hat{T}\}_{i=1}^N$ and add samples to $\mathcal{D}_{\text{model}}$. 
\STATE Update the Q-function $Q_{\phi}$ via Hybrid Value Estimation by solving Eq.~(\ref{eqn:Q_update}) and (\ref{eq_model_update}) iteratively.
\STATE Improve policy $\pi_\theta$ under state marginal of $\mathcal{D}_\text{mix}$ by solving Eq.~(\ref{eq_pi_update}).
\IF {$e$ mod $N_H=0$}
    \STATE Adjust step size $H$ with Eq.~(\ref{eq_adapt_H}).
\ENDIF
\ENDFOR
\end{algorithmic}
\end{algorithm}

\subsection{OPE Metrics}
\label{sec_metric}
We use the following metrics in the OPE experiments, which are the same as those used in~\cite{dope}. $V^{\pi}$ denotes the true value of the policy, and $\hat{V}^{\pi}$ denotes the estimated value of the policy.

\textbf{Absolute Error} Absolute error is the difference between the true value and the estimated value of a policy:

\begin{equation}
    \text{Absolute Error} = \left|V^{\pi} - \hat{V}^{\pi}\right|
\end{equation}

\textbf{Rank Correlation} Rank correlation is an assessment of the degree of correlation between the estimated values and the true values. The definition of rank correlation is:

\begin{equation}
    \text{Rank Correlation} = \frac{\text{Cov}(V^\pi_{1:N}, \hat{V}^\pi_{1:N})}{\sigma(V^\pi_{1:N})\sigma(\hat{V}^\pi_{1:N})}
\end{equation}

\textbf{Regret@k} Regret@k is the difference between the value of the best policy in the entire set, and the value of the best policy in the estimated top-k set. It can be written as:

\begin{equation}
    \text{Regret@k} = \max_{i\in1:N} V^\pi_i - \max_{j\in\text{topk}(1:N)}V^\pi_j
\end{equation}

\subsection{Additional Experimental Setup}\label{sec_setup}

\textbf{Offline Dataset. } 
We make experiments on the D4RL benchmark of Gym-MuJoCo tasks~\cite{gym, d4rl} in both OPE and offline RL experiments, including three environments (hopper, walker2d, and halfcheetah), four dataset types (random, medium, medium-replay, and medium-expert), and therefore twelve problem settings in total. The datasets are generated by different kinds of behavior policies as follows: 
\begin{itemize}
    \item \textbf{random}: roll out a randomly initialized policy for 1M steps.
    \item \textbf{medium}: train a policy online with SAC to approximately 1/3 the performance of the expert, then roll it out for 1M steps. 
    \item \textbf{mixed}: uses the replay buffer of a policy trained up to the performance of the medium agent.
    \item \textbf{medium-expert}: combine samples from a fully-trained policy with samples from a partially trained policy.
\end{itemize}

\textbf{Hyperparameters. }
Here we discuss the hyperparameters that we use for OPHVE and MOHVE. 
\begin{itemize}
    \item \textbf{Model rollout length $h$. }Similar to most of the model-based algorithms~\cite{mbpo, mopo, combo}, model rollout length $h$ is an important hyperparameter in our experiments. Note that different from the hybrid step length $H$, $h$ is the length of rollout in the learned model. We use $h=10$ in all tasks except for hopper-mixed and walker2d-mixed, unlike prior model-based methods~\cite{mbpo, mopo, combo}, whose maximum rollout length is $5$. We hypothesize the longer rollout length of HVE owing to its ability to reduce the value estimation error induced by large $h$.
    \item \textbf{Maximum of the hybrid step length $H_{\max}$. }To avoid too large step length $H$, which is computed through minimizing the approximated error bound, we clip $H$ with parameter $H_{\max}$. We find this parameter robust, and set it to $4$ in 10 out of 12 tasks.
    \item \textbf{KL regularization parameter $\delta$. }The KL divergence regularization (cf.~Eq.~(\ref{eq_pi_update})) which controls the policy divergence term $\epsilon_\pi$. 
    \item \textbf{Real data ratio $\alpha$. }This parameter determines the state distribution $\mathcal{D}_{\text{mix}}$ used in the policy improvement stage. We search for $\alpha \in \{0.2, 0.5, 0.8\}$.
\end{itemize}


The four mentioned hyper-parameters for each task are listed in Tab.~\ref{tab: hyper_parameter}.

\begin{table}[htb]
\centering
\caption{Hyperparameters used in experiments.}
\label{tab: hyper_parameter}
\begin{tabular}{@{}l|l|l|l|l|l@{}}
\toprule
\textbf{Data Type} & \textbf{Env.} & \textbf{$h$} & \textbf{$H_{\max}$} & \textbf{$\delta$} & \textbf{$\alpha$} \\ \midrule
random  & halfcheetah & 10 & 4 & 20  & 0.2 \\
random  & hopper      & 10 & 4 & 10  & 0.8 \\
random  & walker2d    & 10 & 4 & 10  & 0.5 \\
medium  & halfcheetah & 10 & 4 & 20  & 0.5 \\
medium  & hopper      & 10 & 4 & 10  & 0.2 \\
medium  & walker2d    & 10 & 4 & 5   & 0.8 \\
mixed   & halfcheetah & 10 & 1 & 20  & 0.2 \\
mixed   & hopper      & 3  & 4 & 50  & 0.5 \\
mixed   & walker2d    & 1  & 2 & 8   & 0.5 \\
med-exp & halfcheetah & 10 & 4 & 5   & 0.2 \\
med-exp & hopper      & 10 & 4 & 0.5 & 0.2 \\
med-exp & walker2d    & 10 & 4 & 5   & 0.5 \\ \bottomrule
\end{tabular}
\end{table}

\textbf{Model Training and Usage. } We train an ensemble of $7$ models and use $5$ models with a smaller validation error. Each model is parameterized as a 4-layer feedforward neural network with $256$ hidden units. And there are two heads outputting the mean and logarithmic standard deviation of the next state. We also apply standardization to the input state and reward of the model. During each epoch, we sample $50000$ states from the offline dataset as the initial state of each model rollout trajectory. 

For OPHVE, we perform $100$k gradient updates to learn the Q function.
For MOHVE, we set $N_H$, the frequency of updating $H$, to $200$ across all tasks, considering the stability of training. Other hyperparameters that correspond to the backbone RL algorithm SAC like learning rate and gradient steps follow the setup of COMBO~\cite{combo}.

\subsection{Computational Resources}

All experiments are conducted on a single NVIDIA GeForce RTX 2080 Ti. OPHVE takes about $1.5$ hours to estimate $10$ policies. And it takes about $10$ hours to train MOHVE for $1000$ epochs, including the time of online evaluation every epoch.

\subsection{License of datasets}
We acknowledge that D4RL datasets use the MIT license.

\section{Additional Results}
\label{additional results}

\subsection{OPE Experiments}
\label{additional ope results}
Detailed OPE results on each task are shown in Fig.~\ref{fig:add ope}.
\newpage

\begin{figure}[H]
    \centering
    \subfigure{
    \includegraphics[width=0.98\textwidth]{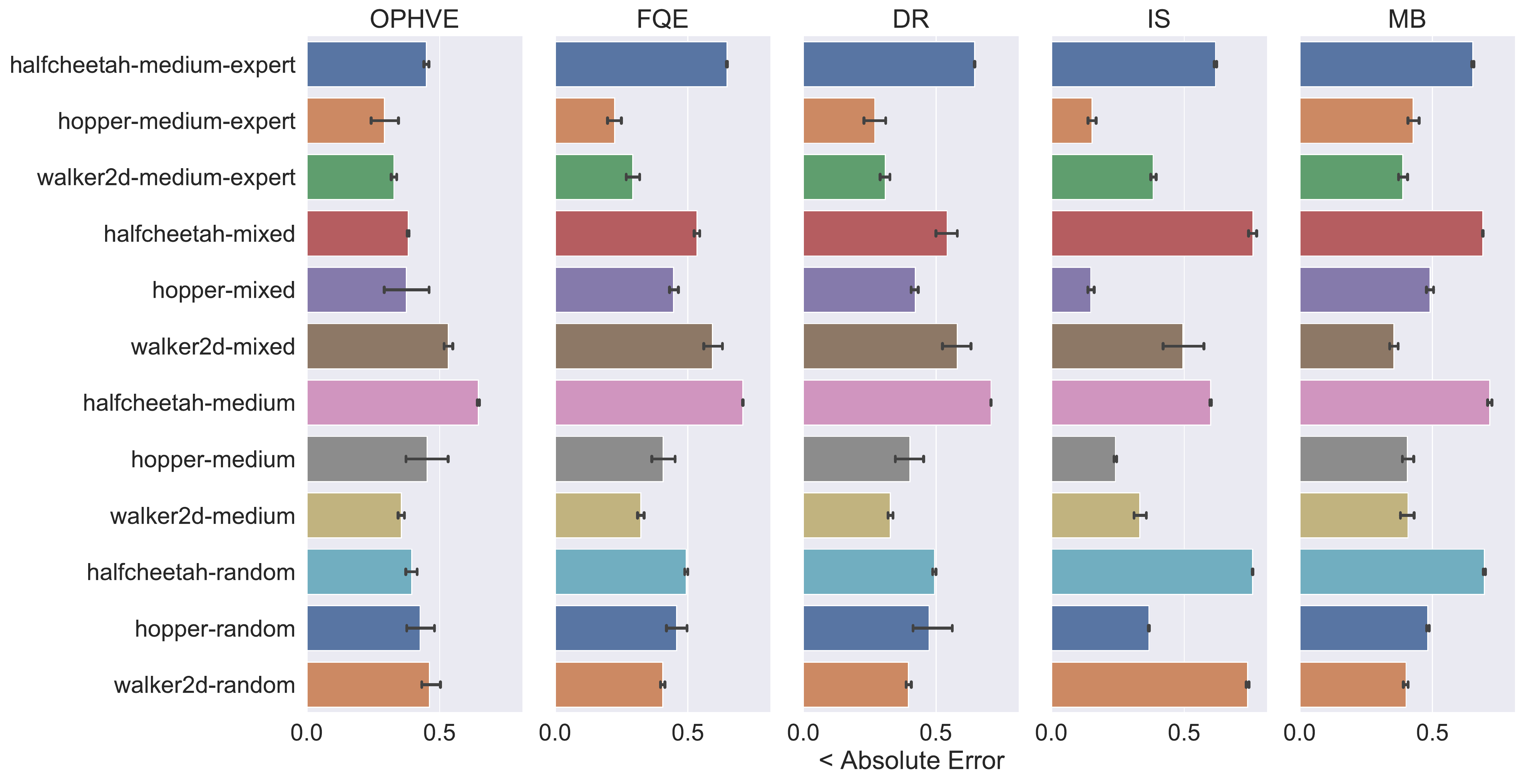}
    }

    \subfigure{
    \includegraphics[width=0.98\textwidth]{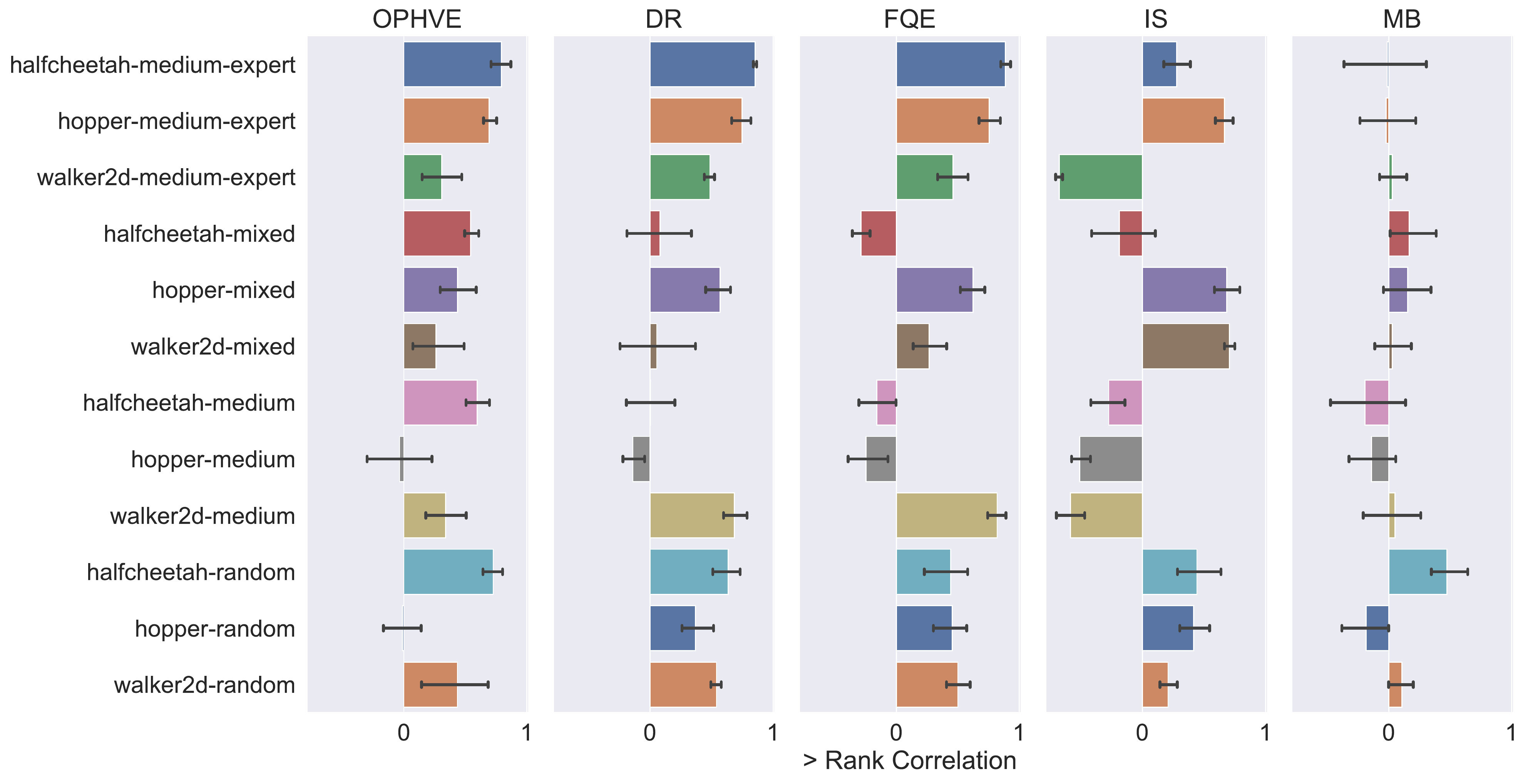}
    }
    
    \subfigure{
    \includegraphics[width=0.98\textwidth]{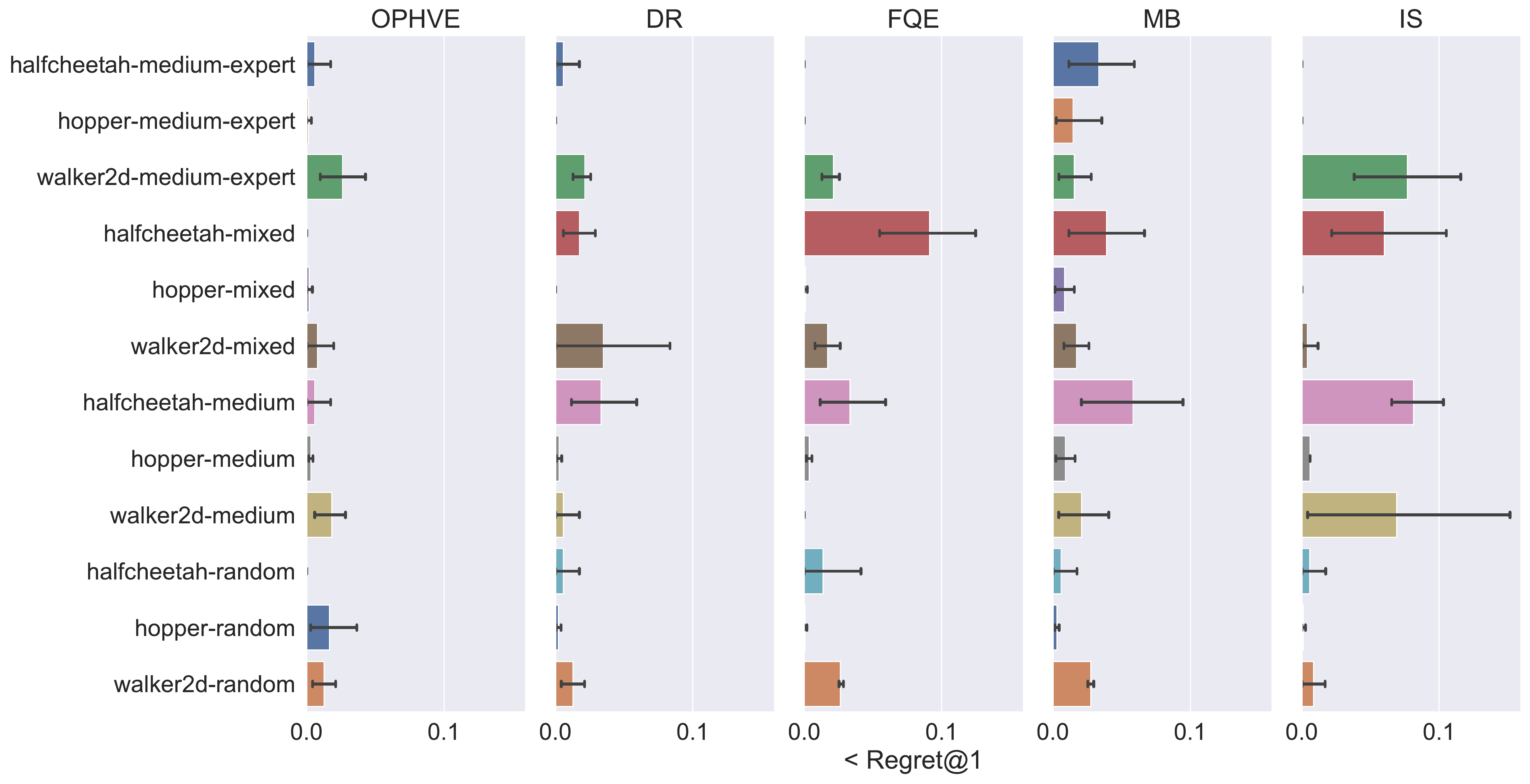}
    }
    \caption{\label{fig:add ope}OPE results on each task.}
\end{figure}

\newpage
\subsection{Offline RL Experiments}
\label{additional offline RL results}

We plot the learning curves of MOHVE on all tasks. We repeat every experiment five times in different random seeds. All results are shown in Fig.~\ref{fig:additional offline RL}. MOHVE converges to relatively high scores in most datasets, which demonstrates the effectiveness of our Hybrid Value Estimation method in offline RL. Furthermore, the training process is stable with a small variance. 

\begin{figure}[ht]
    \centering
    \subfigure{
        \includegraphics[width=0.32\textwidth]{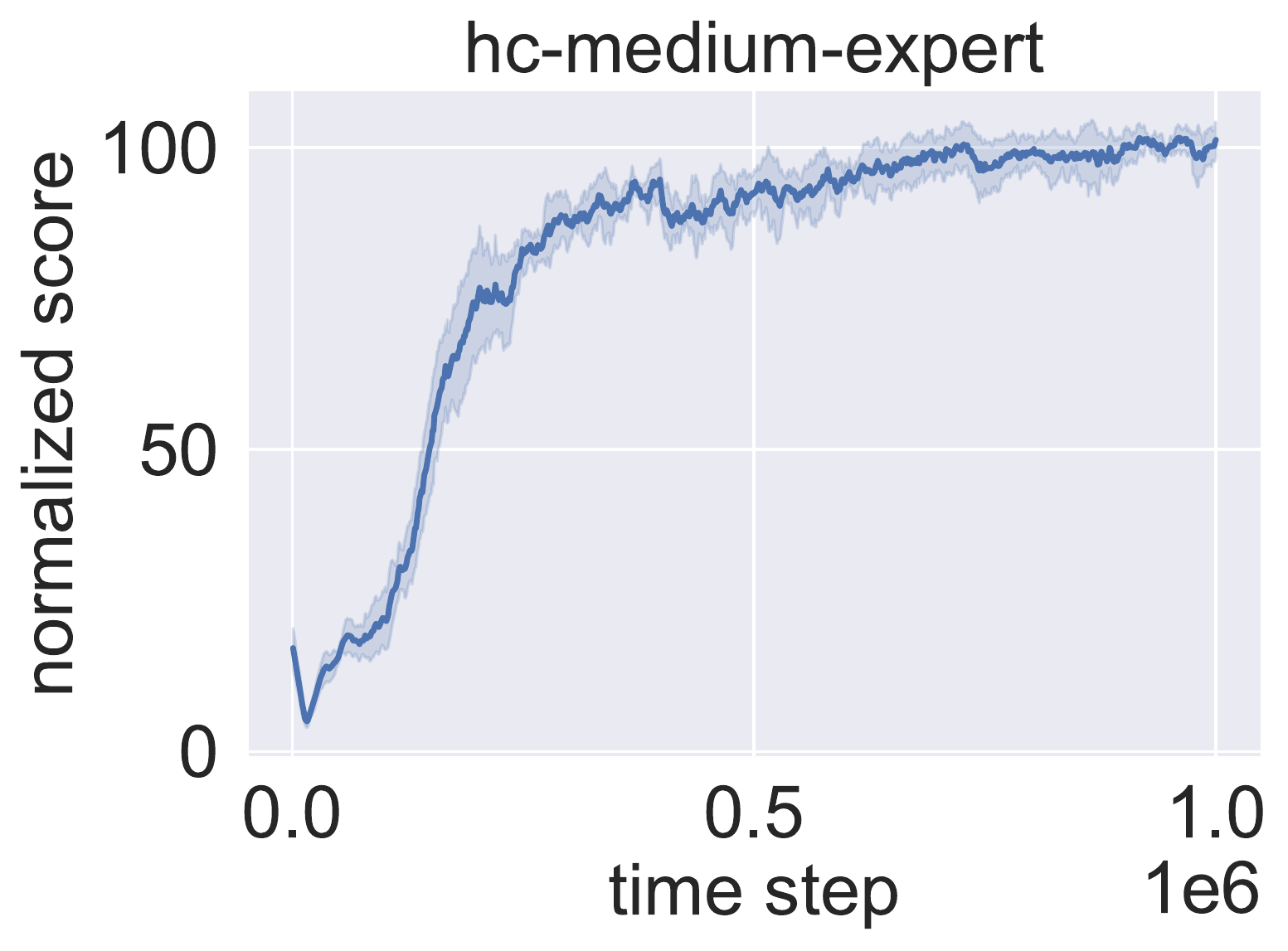}
    }
    \subfigure{
        \includegraphics[width=0.32\textwidth]{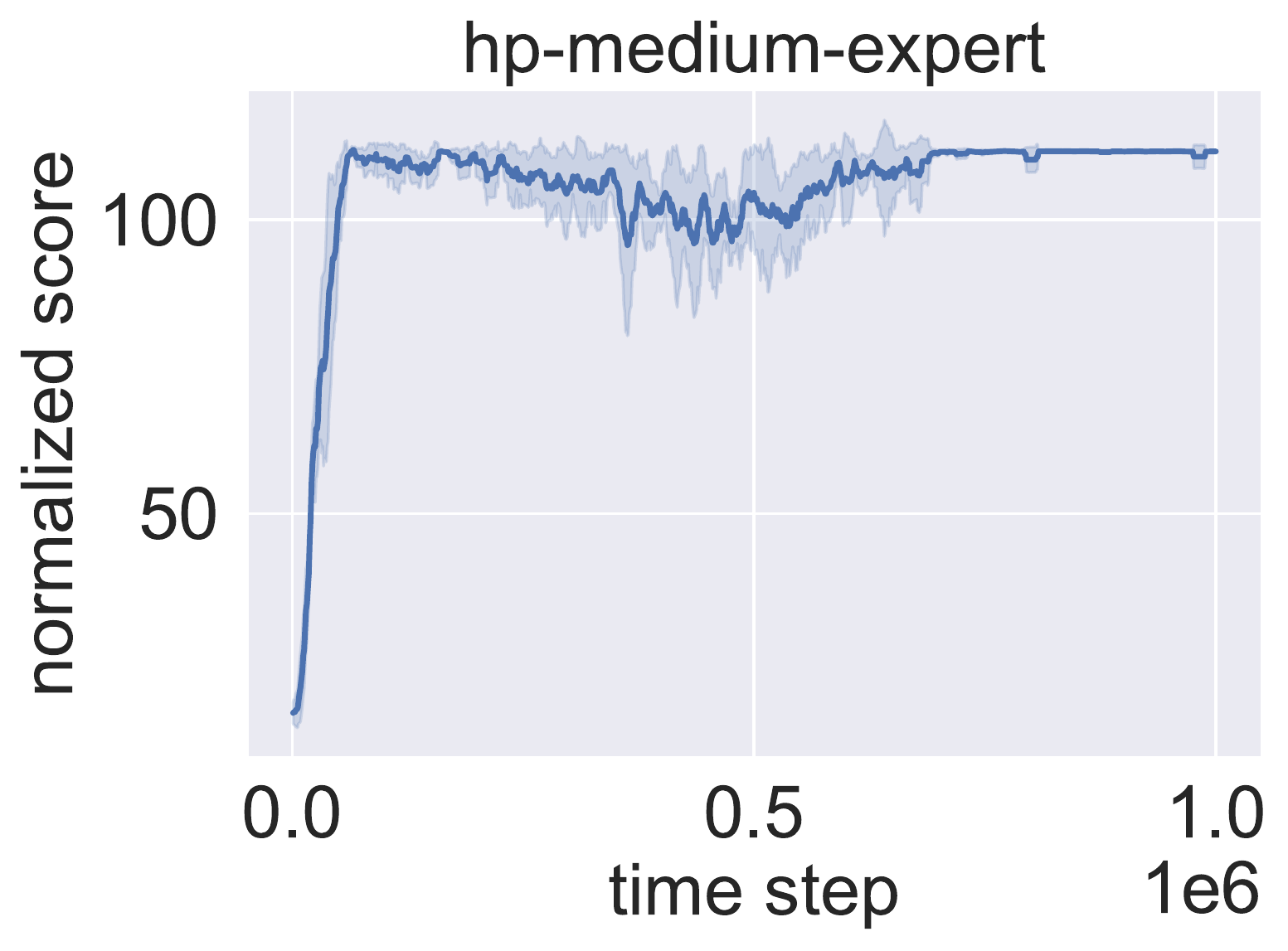}
    }
    \subfigure{
        \includegraphics[width=0.31\textwidth]{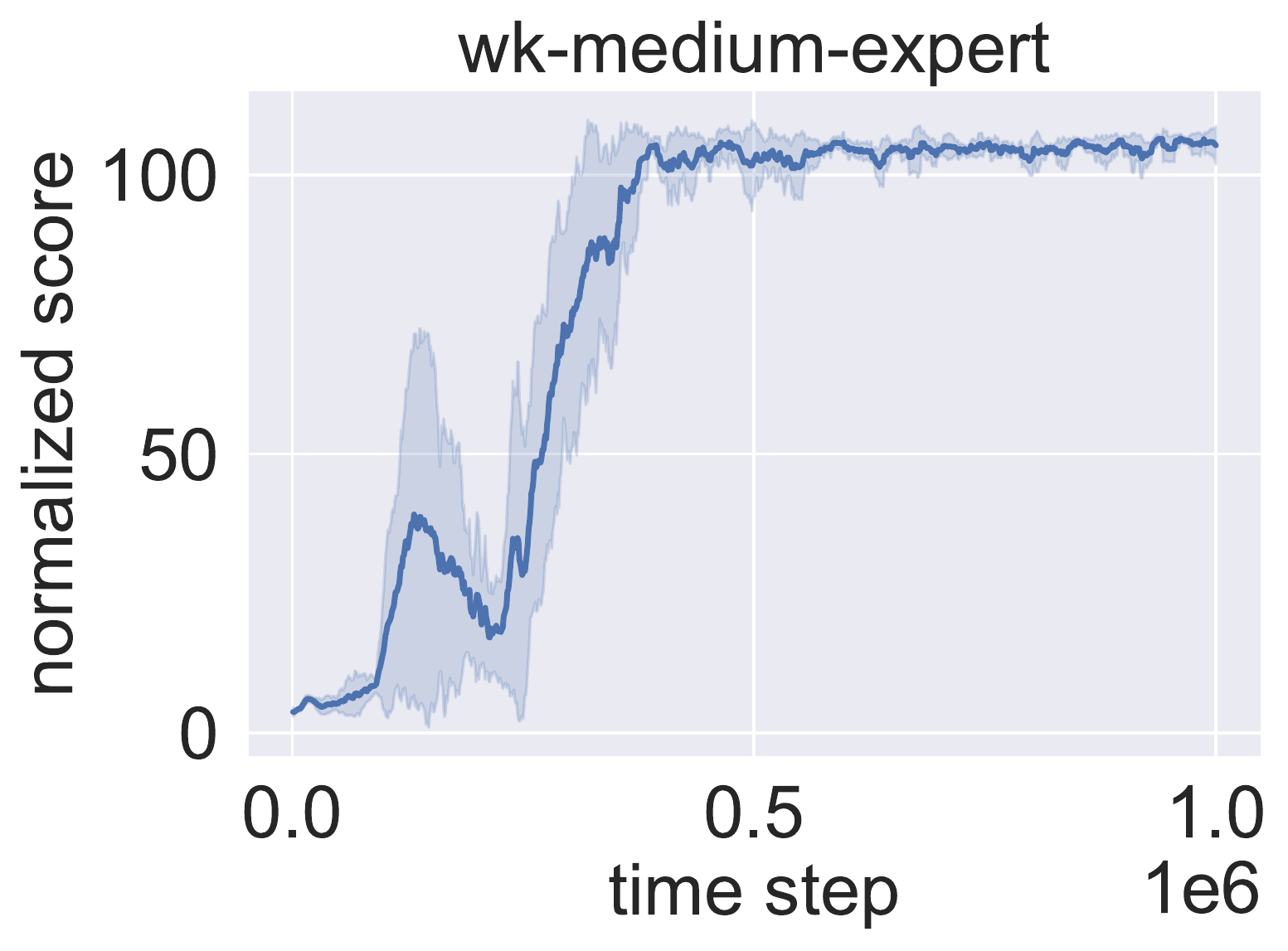}
    }\\
    \subfigure{
        \includegraphics[width=0.32\textwidth]{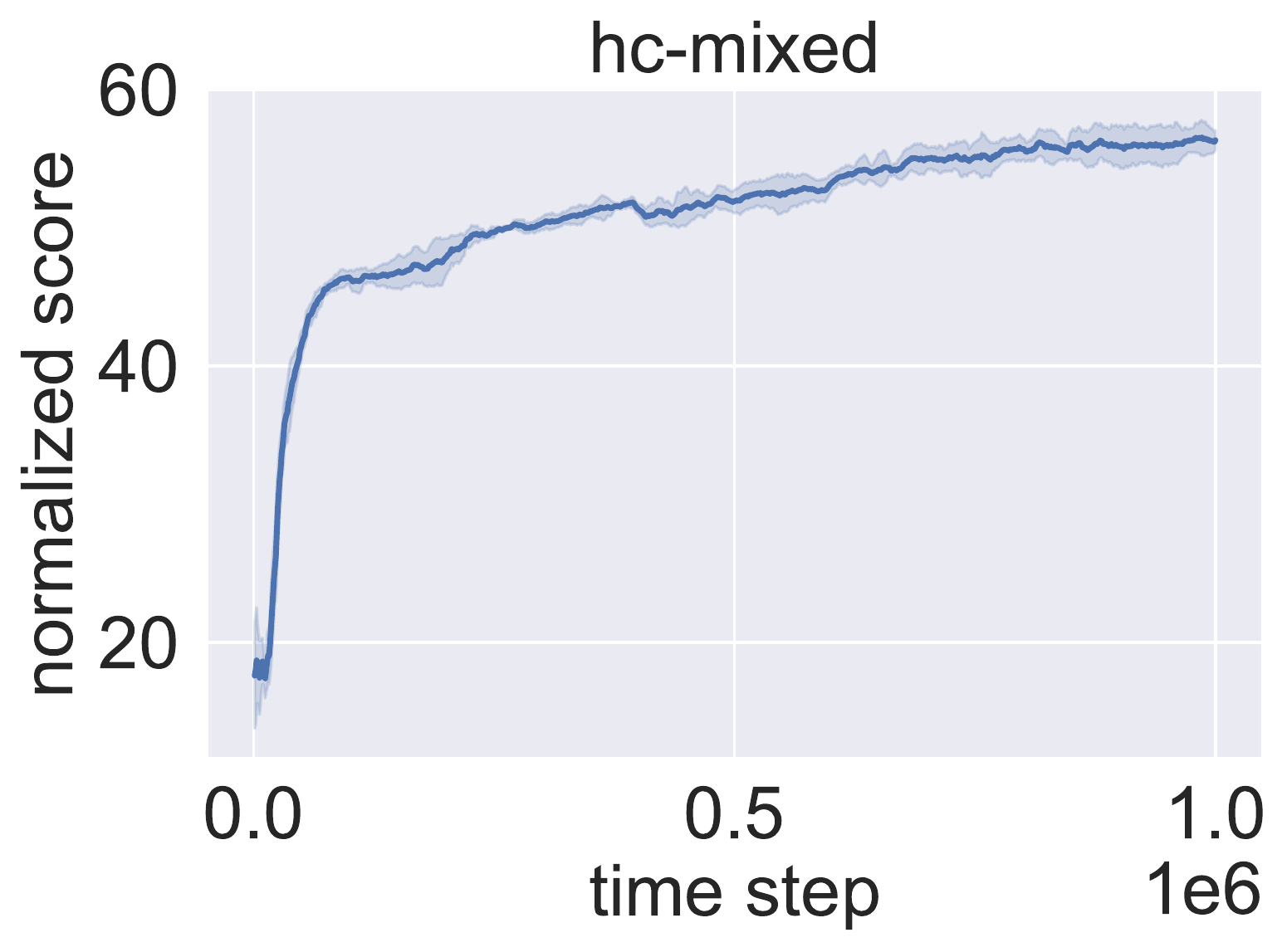}
    }
    \subfigure{
        \includegraphics[width=0.32\textwidth]{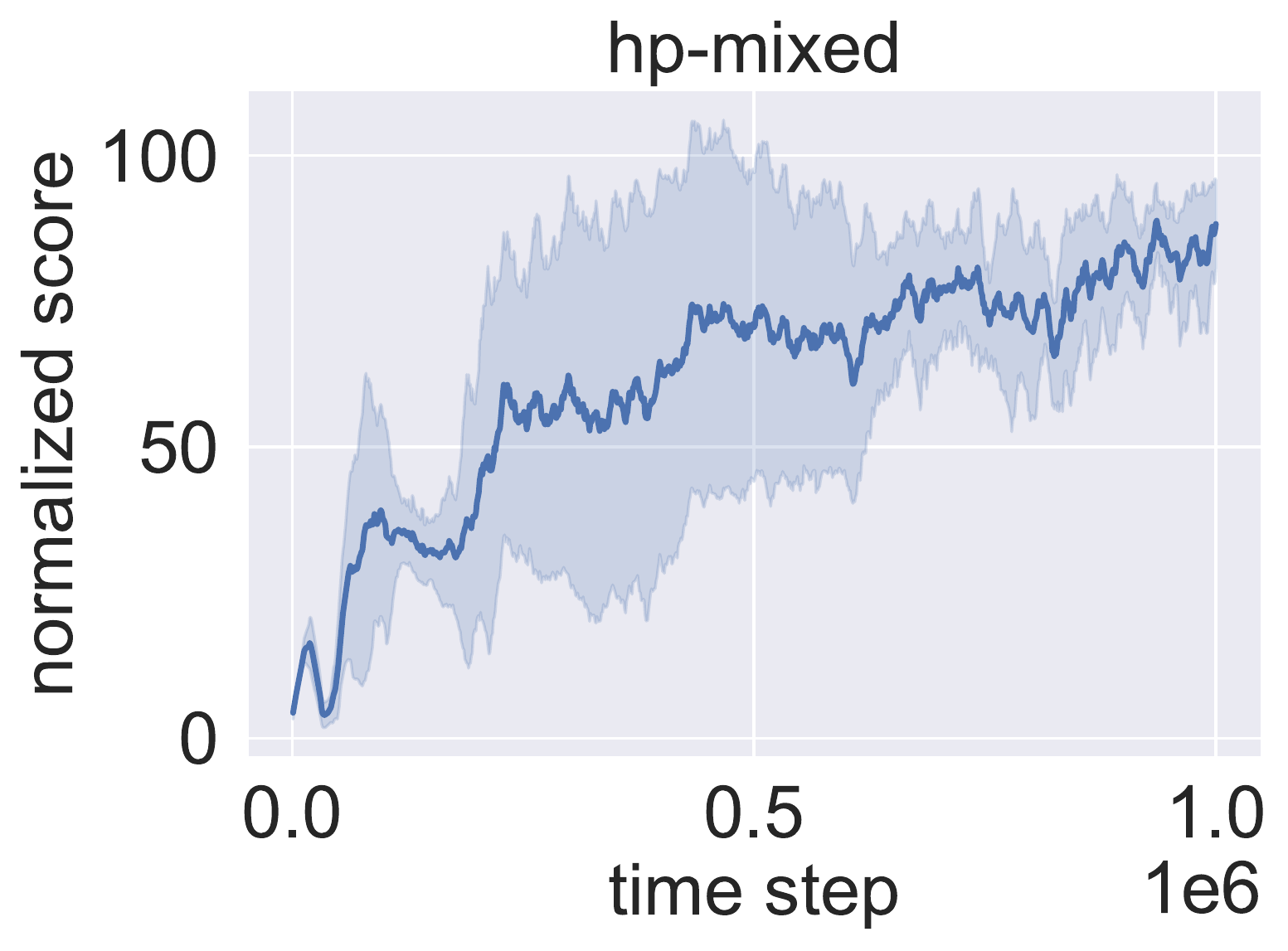}
    }
    \subfigure{
        \includegraphics[width=0.31\textwidth]{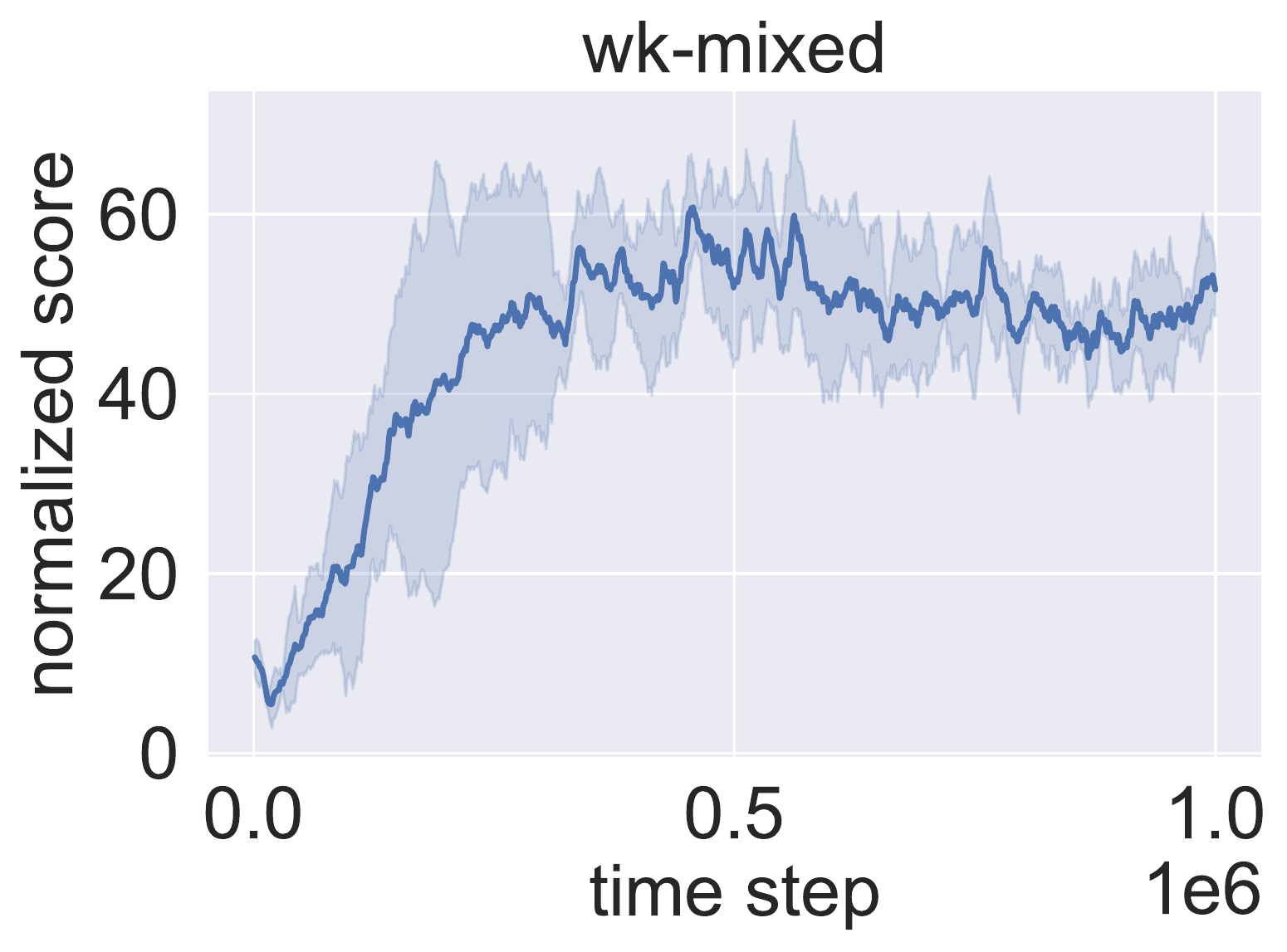}
    }\\
    \subfigure{
        \includegraphics[width=0.32\textwidth]{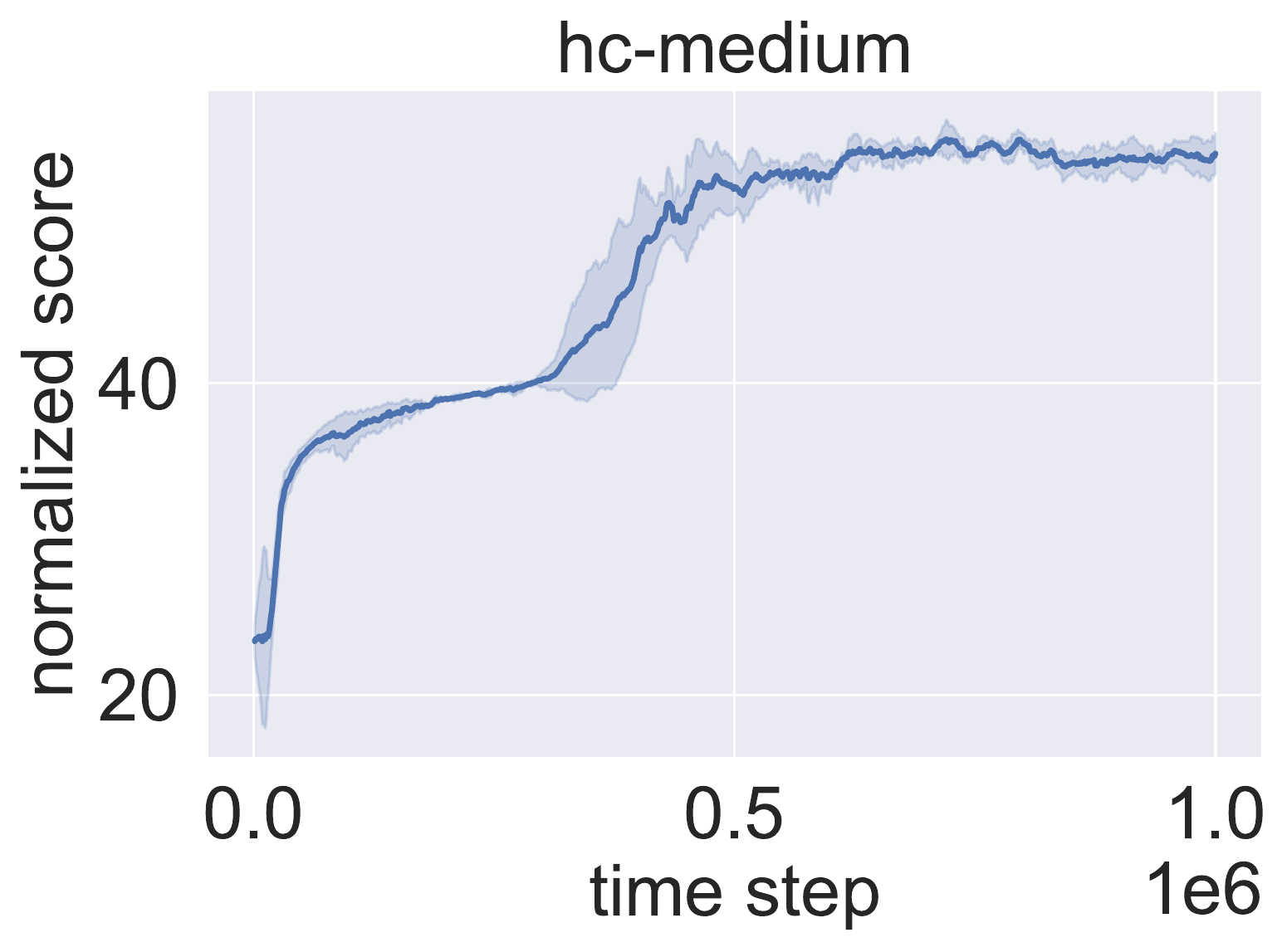}
    }
    \subfigure{
        \includegraphics[width=0.32\textwidth]{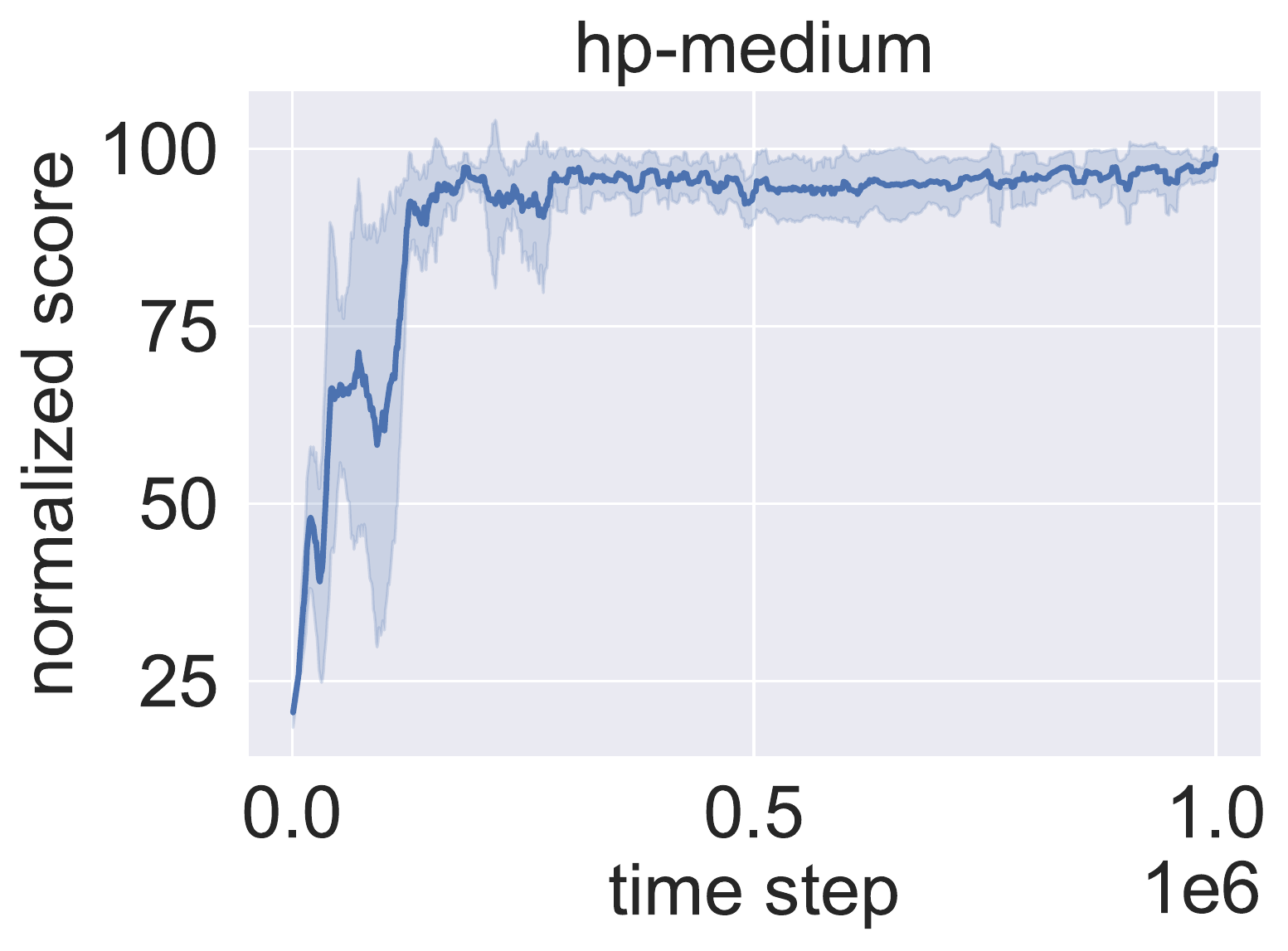}
    }
    \subfigure{
        \includegraphics[width=0.31\textwidth]{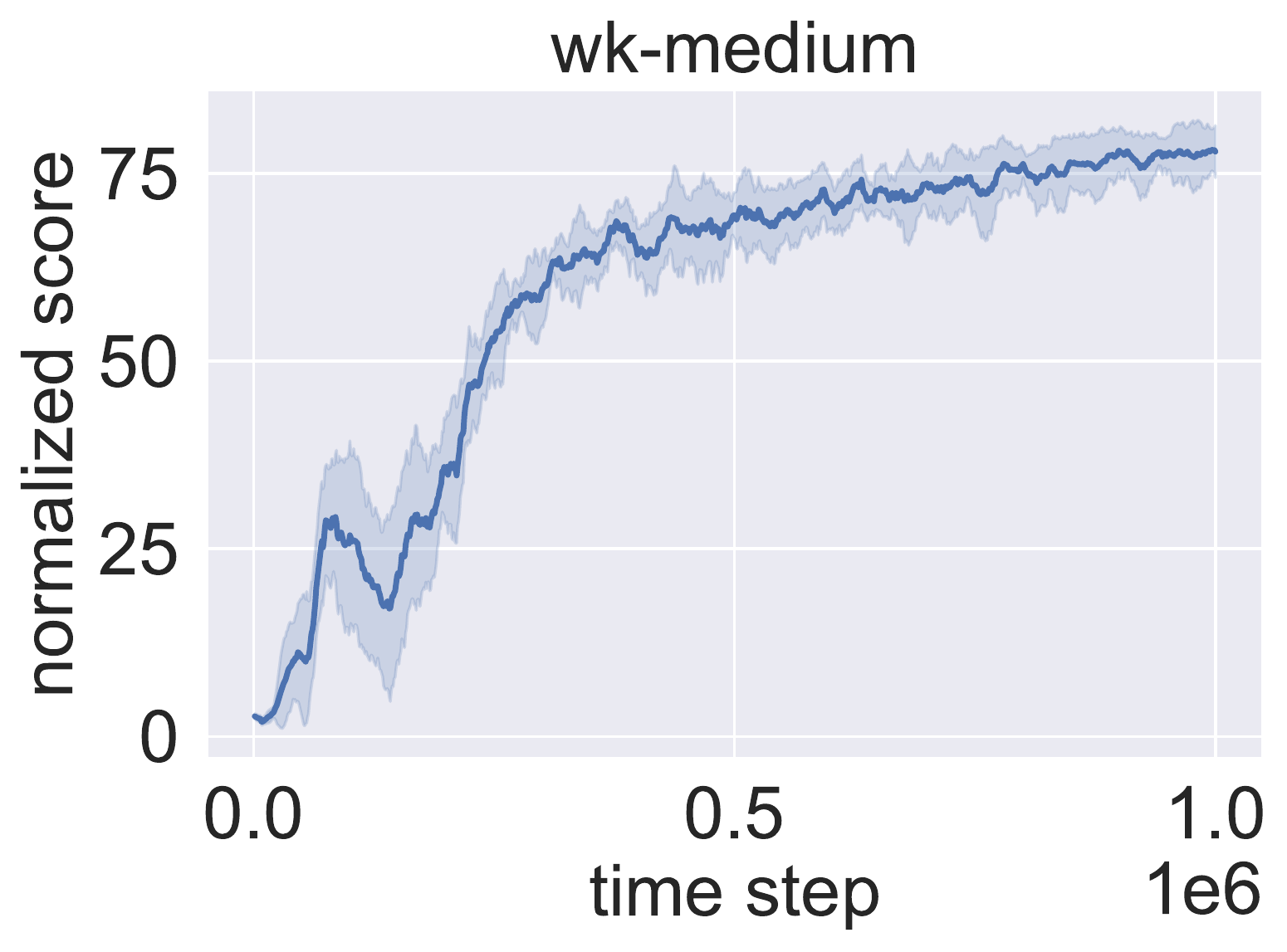}
    }\\
    \subfigure{
        \includegraphics[width=0.32\textwidth]{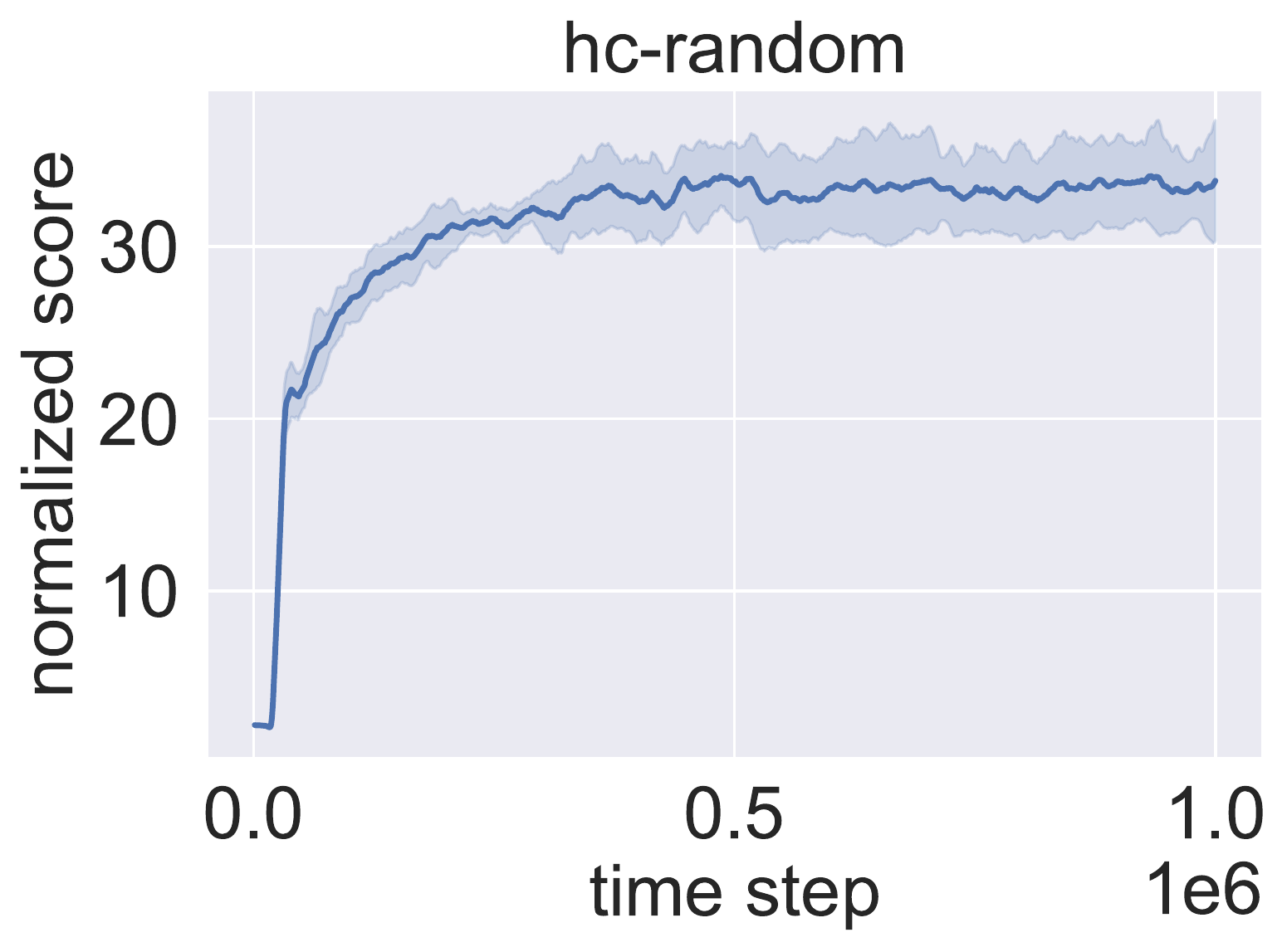}
    }
    \subfigure{
        \includegraphics[width=0.32\textwidth]{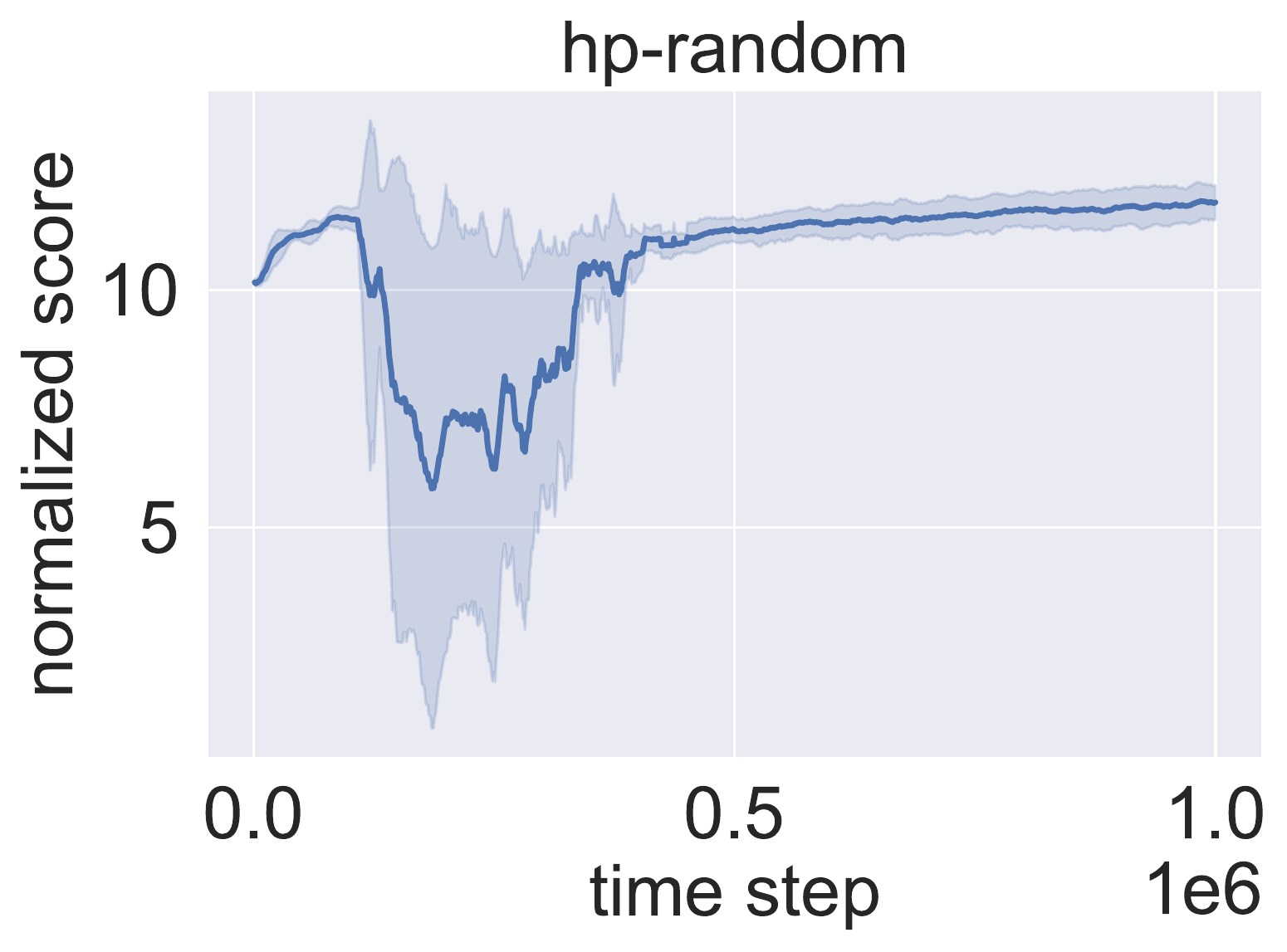}
    }
    \subfigure{
        \includegraphics[width=0.31\textwidth]{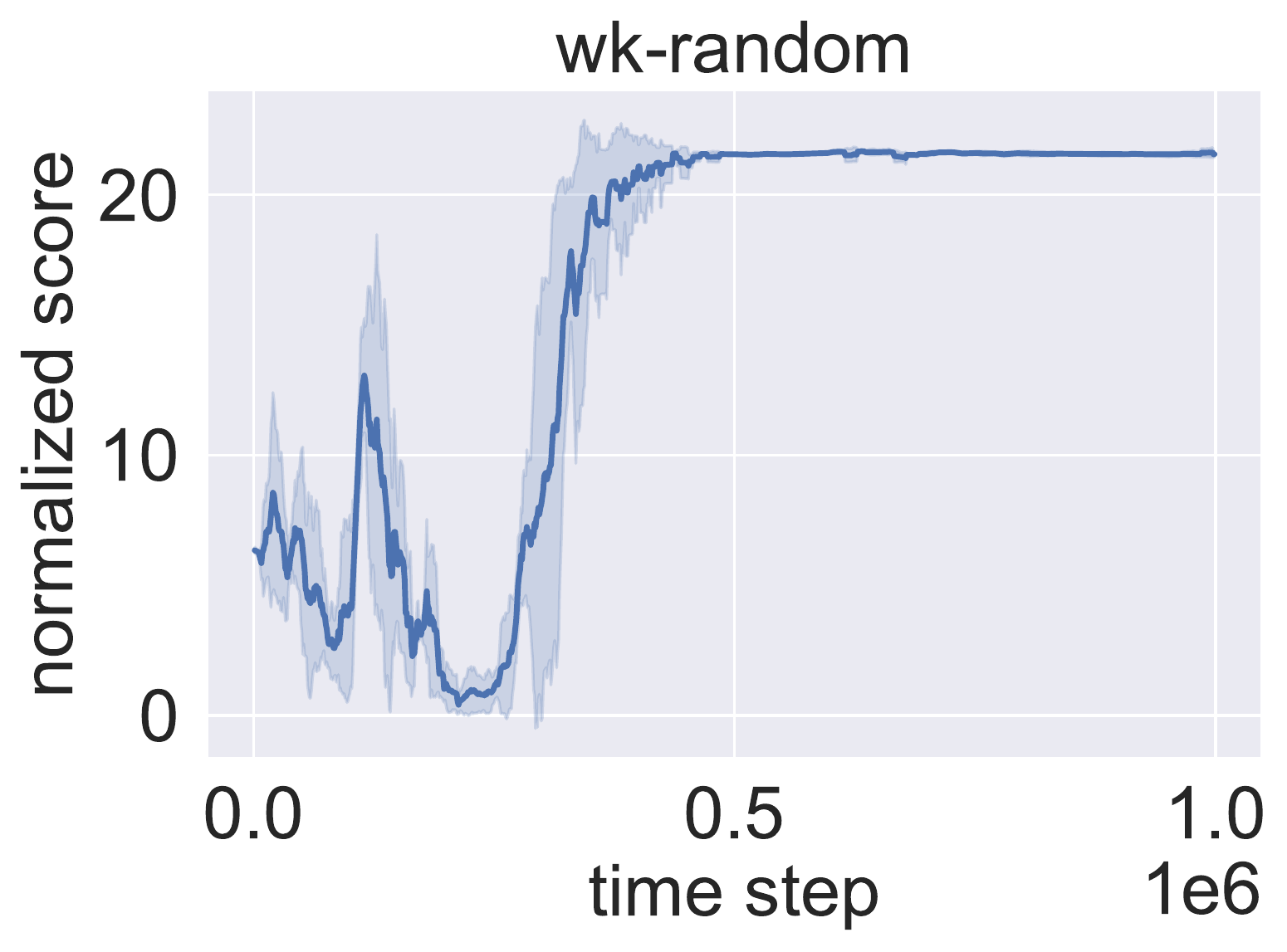}
    }
    
    \caption{Learning curves of MOHVE on all tasks, where hc, hp and wk represent halfcheetah, hopper and walker2d respectively. The solid curves are the mean of normalized return and shadow is the standard error. The performance of MOHVE is stable and robust in most of the tasks.}
    \label{fig:additional offline RL}
\end{figure}

\newpage
\subsection{Analysis of step length $H$ on walker2d-medium-expert}
\label{sec_additional_analysis_of_H}

We analyze the effect of step length $H$ on OPHVE and MOHVE on walker2d-medium-expert task. The results of OPE are shown in Fig.~\ref{fig: ope_analysis}. When $H=0$, chosen by the automatic adjustment mechanism, OPHVE performs better than other choices on rank correlation and regret@1 and only performs no better than $H=-1$ on absolute error. Fig.~\ref{fig: offline_rl_analysis} shows the learning curves of MOHVE with automatic adjusted $H$ (denoted by MOHVE) and fixed $H$. MOHVE converges to the highest normalized score with respect to other chosen step lengths. The two experiments verify that the automatic adjustment mechanism contributes to the performance of MOHVE. 
\begin{figure}[h]
    \centering
    \subfigure[]{
    \includegraphics[width=0.3\textwidth]{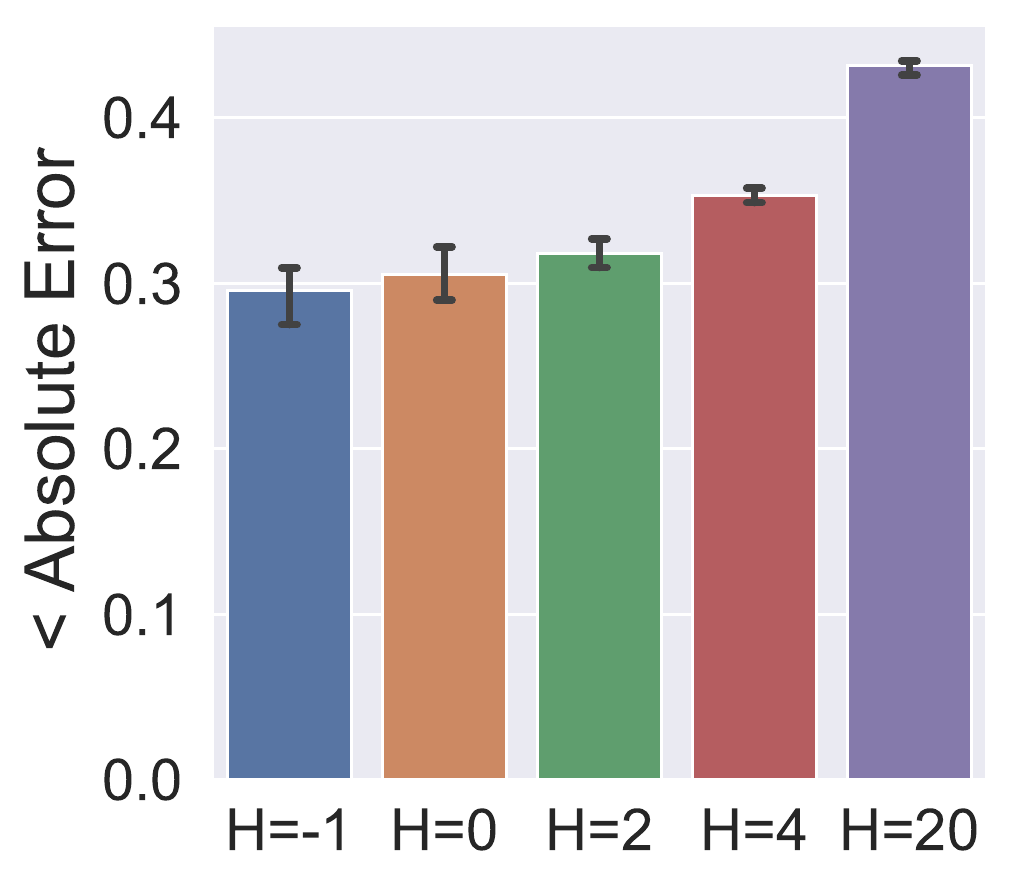}
    }
    \subfigure[]{
    \includegraphics[width=0.3\textwidth]{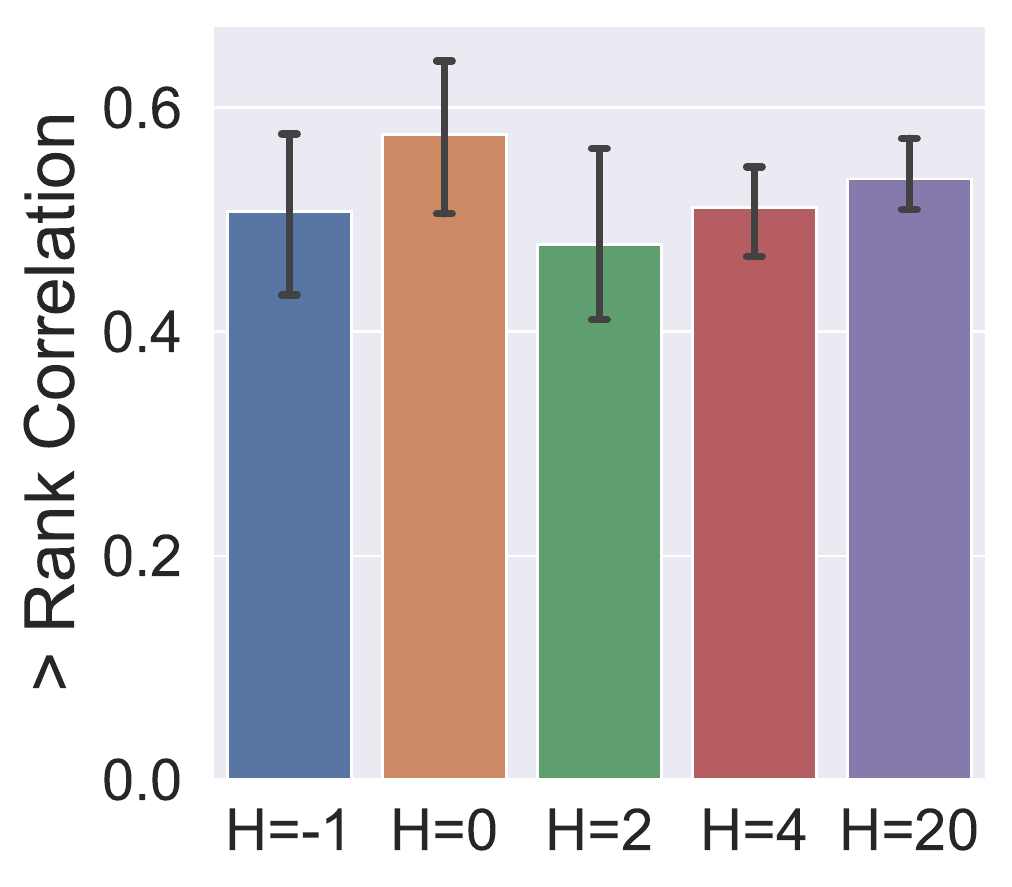}
    }
    \subfigure[]{
    \includegraphics[width=0.3\textwidth]{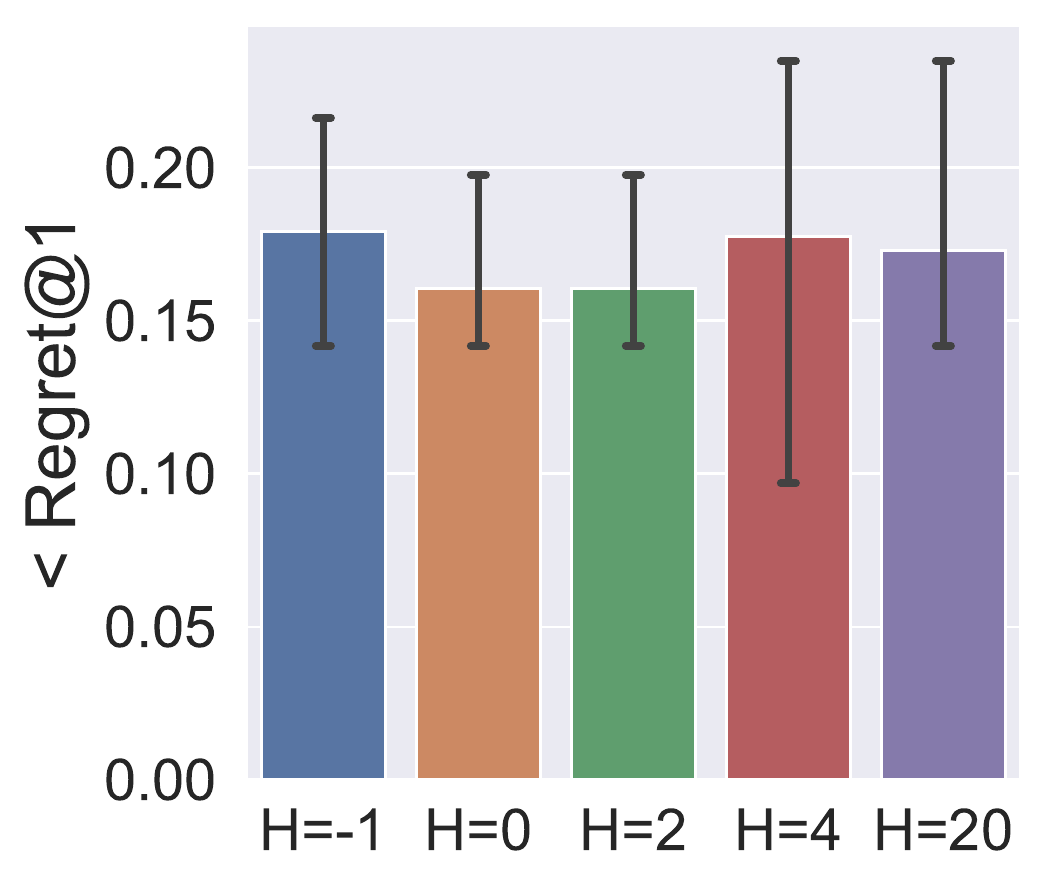}
    }
    \caption{\label{fig: ope_analysis}OPE results on walker2d-medium-expert with different step length $H$.}
\end{figure}
\begin{figure}[h]
    \centering
    \includegraphics[width=0.5\textwidth]{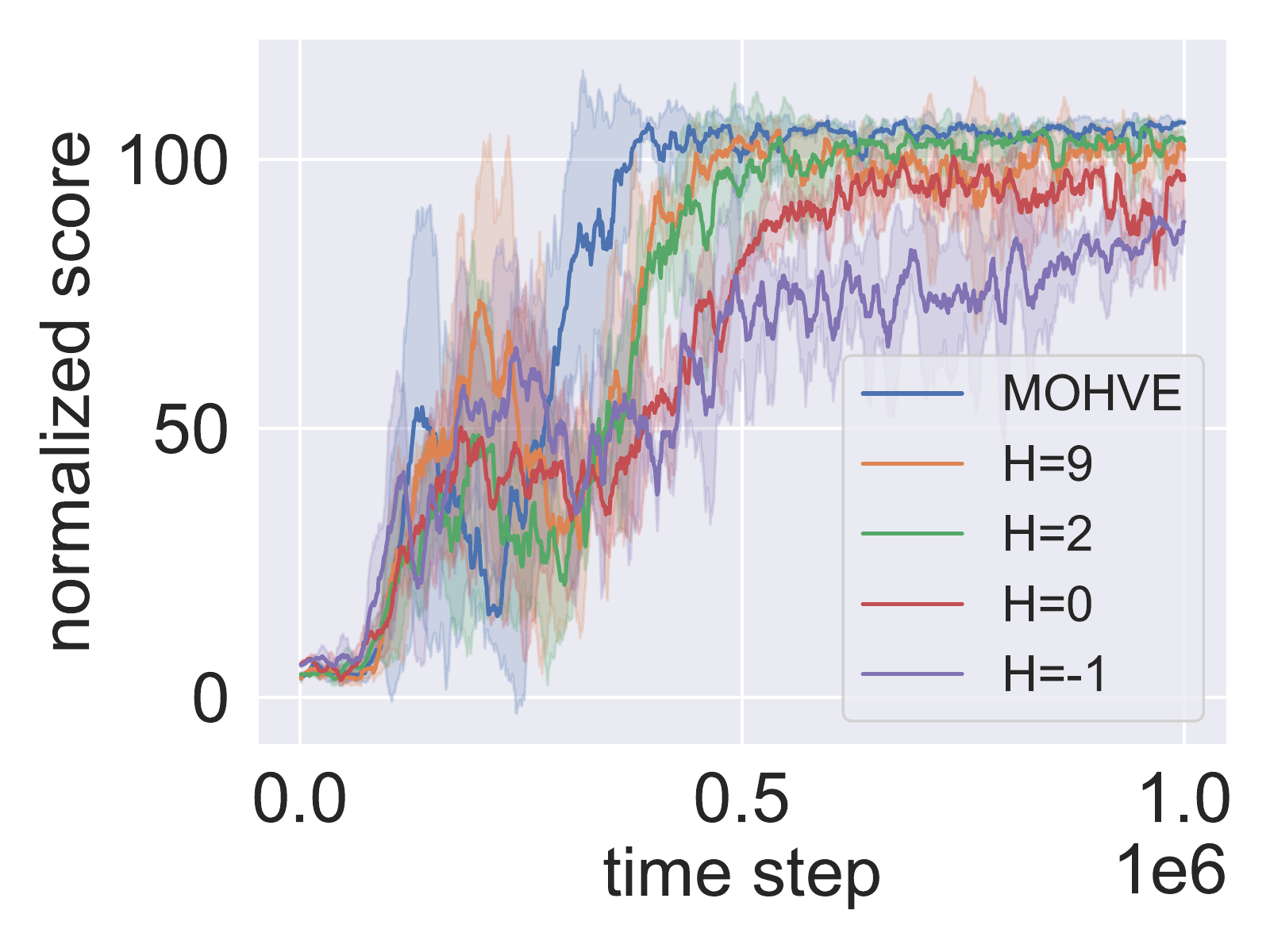}
    \caption{Learning curves of offline RL with different step length $H$. MOHVE with automatic adjusted step length $H$ obtains the highest score.}
    \label{fig: offline_rl_analysis}
\end{figure}

\end{document}